\documentclass[twoside]{article}
\usepackage[left=2cm,right=2cm,top=2cm,bottom=2cm]{geometry}
\usepackage{amsmath}
\usepackage{amsfonts}
\usepackage{amssymb}
\usepackage{graphicx}
\usepackage{color}
\usepackage[normalem]{ulem}
\usepackage[utf8]{inputenc}
\usepackage{amsfonts}
\usepackage{color}
\usepackage[normalem]{ulem}
\usepackage{nomencl}

\usepackage{algorithm}
\usepackage{algpseudocode}

\makenomenclature

\newenvironment{proof}{\noindent{\sc Proof.}}{\qed}
\newtheorem{theorem}{Theorem}[section]
\newtheorem{lemma}{Lemma}[section]
\newtheorem{cor}{Corollary}[section]
\newtheorem{rem}{Remark}[section]
\newtheorem{definition}{Definition}[section]

\newtheorem{uda}{Example}[section]
\newcommand{\qed}{$\blacksquare$}

\def\bhag#1{\noindent
\setcounter{equation}{0}
\section{#1}
}

\def\RR{{\mathbb R}}

\def\ZZ{{\mathbb Z}}

\def\bs#1{{\boldsymbol{#1}}}

\def\argmin{\mathop{\hbox{\textrm{arg min}}}}

\def\be{\begin{equation}}
\def\ee{\end{equation}}
\def\bea{\begin{eqnarray}}
\def\eea{\end{eqnarray}}
\def\eref#1{(\ref{#1})}

\def\donchitre#1#2{\vskip 6.5cm\noindent
\parbox[t]{1in}{\special{eps:#1.eps x=6.5cm y=5.5cm}}
\hbox to 7cm{}\parbox[t]{0.0cm}{\special{eps:#2.eps x=6.5cm y=5.5cm}}}

\def\tn{|\!|\!|}
\def\XX{{\mathbb X}}

\def\bs#1{{\boldsymbol{#1}}}

\usepackage{dsfont}

\title{A low discrepancy sequence on graphs}
\author{
 A.~Cloninger\thanks{Department of Mathematics and Halicioglu Data Science Institute, University of California San Diego, San Diego, CA 92093, U.S.A..  
 \textsf{email:} acloninger@ucsd.edu},\ \  H.~N.~Mhaskar
 \thanks{
Institute of Mathematical Sciences, Claremont Graduate University, Claremont, CA 91711, U.S.A.. 
\textsf{email:} hrushikesh.mhaskar@cgu.edu} 
  }
\date{}
\begin{document}
\maketitle

\begin{abstract}
Many applications such as election forecasting, environmental monitoring, health policy, and graph based machine learning require taking expectation of functions defined on the vertices of a graph. 
We describe a construction of a sampling scheme analogous to the so called Leja points in complex potential theory that can be proved to give low discrepancy estimates for the approximation of the expected value by the impirical expected value based on these points.
In contrast to classical potential theory where the kernel is fixed and the equilibrium distribution depends upon the kernel, we fix a probability distribution and construct a kernel (which represents the graph structure) for which the equilibrium distribution is the given probability distribution. 
Our estimates do not depend upon the size of the graph.
\end{abstract}


\bhag{Introduction}\label{bhag:introduction}

In many applications, the data is  not represented by points in a high dimensional Euclidean space, but instead as vertices of  networks with pairwise relations.  
A common problem that arises is estimating the mean or some integral of a function on those vertices, under the assumption that the function is  smooth in an appropriate sense with respect to the network.  
For example, prior to elections, polls are used to estimate the opinions of a networked
population about the candidates in expectation \cite{gayo2013meta}.  
In environmental monitoring, estimating average temperature or water quality from a network of sensors allows governments to take preventative action \cite{krause2008efficient}. 
 In health policy, monitoring the average level of health and happiness in various populations requires extensive random polling \cite{sears2014well, cloninger2019people}, which could be vastly downsampled with knowledge of the social networks of the population.  
 Finally, in semi-supervised and active learning using Graph Neural Networks, the generalization error is specified as the average of the loss function associated with all vertices in a graph \cite{kipf2016semi, wu2019active}.

In real world networks, evaluating the function at every vertex, or even at a large number of vertices, is very impractical.  Therefore, an important question is how to choose a smaller subset of the vertices at which to evaluate the function and estimate the desired integral.  Approaches such as Monte-Carlo integration are independent of the network, and do not exploit the assumption that the function is smooth with respect to the network.  
Methods that exploit the geometry of the domain, including potential theory estimates \cite{erd1940uniformly, blatt1992distribution, andrievskii2013discrepancy} and quasi-Monte-Carlo sampling \cite{dick2010digital}, are only defined for intervals and compact vector spaces, and only for particular measures.  

The purpose of this paper is to investigate deterministic constructions of low discrepancy sequences on undirected graphs, where the integration is taken with respect to an arbitrary measure, supported on the entire vertex set of the graph.
We will use potential theory ideas, in particular, study a construction analogous to the so called Leja points.   
In contrast to classical potential theory, we start with a given measure and a graph, and modify the weights on the graph so that the given measure is the equilibrium measure.

\subsection{Related work}\label{bhag:relatedwork}

Let $\Omega$ be any measure space, $\mu$ be a probability measure on $\Omega$, $f$ be a random variable on $\Omega$ with $|f(x)|\le R$ for almost all $x\in\Omega$, $\mathbb{E}_\mu(|f|^2)=V$. 
Then if $\delta\in (0,1)$,  $M$ is sufficiently large, and $x_1,\cdots,x_M$ are random samples from $\Omega$, then it is well known (for example, using Bernstein concentration inequality) that with $\mu$-probability $\ge 1-\delta$,
\be\label{montecarloest}
\left|\frac{1}{M}\sum_{k=1}^M f(x_k)-\int_\Omega f(x)d\mu(x)\right| \le \sqrt{\frac{2(V+R)\log(2/\delta)}{M}}.
\ee
Of course, deterministic variants of this inequality can be obtained using different kinds of deterministic assumptions on $f$. 
For example, an equivalent formulation of a theorem of Erd\H os and Tur\'an  \cite{erd1940uniformly} is the following. Let $\{P_M(z)=\prod_{k=1}^M (z-x_{k,M})\}$ be a sequence of monic polynomials, and $A_M=2^M\max_{z\in [-1,1]}|P_M(z)|$. 
Then for any function $f :[-1,1]\to\RR$ having a bounded total variation $\|f\|_{TV}$ on $[-1,1]$, we have
\be\label{erdosturan}
\left|\frac{1}{M}\sum_{k=1}^M f(x_{k,M})-\frac{1}{\pi}\int_{-1}^1 \frac{f(x)}{\sqrt{1-x^2}}dx\right|\le \frac{8}{\log 3}\sqrt{\frac{\log A_M}{M}}\|f\|_{TV}.
\ee
Theorems of this kind are known as discrepancy theorems. 
Although this estimate cannot be improved in general, 
it was improved by Blatt \cite{blatt1992distribution} with additional assumptions as follows. 
Let the points $\{x_{k,M}\}\subset [-1,1]$, and $2^M|P_M'(x_{k,M})|\ge B_M^{-1}$. 
Then there exists a  contant $c>0$ such that
\be\label{blattdiscr}
\left|\frac{1}{M}\sum_{k=1}^M f(x_{k,M})-\frac{1}{\pi}\int_{-1}^1 \frac{f(x)}{\sqrt{1-x^2}}dx\right|\le c\log\max(A_M,B_M,M)\frac{\log M}{M}\|f\|_{TV}.
\ee 
Both of these estimates depend heavily on potential theory estimates; in particular, the fact that $1/2$ is the logarithmic capacity of $[-1,1]$, and the measure of integration above is the equilibrium distribution for $[-1,1]$. 
There a is large amount of research devoted to generalization of this work including those involving potential theory in higher dimensions. 
In particular, an analogue of \eqref{blattdiscr} in the case of arbitrary measures rather than equilibrium measures is given in \cite{blatt1993general}.
A survey can be found in the book \cite{ andrievskii2013discrepancy} of Andrievskii and Blatt.

The problem is of interest also in the theory of information based complexity where one seeks to approximate an integral over high dimensional spaces with averages of samples of the integrand. 
These have a different flavor, where instead of thinking in terms of zeros of polynomials and potential theory arguments, the interest is in devising quasi-Monte-Carlo systems with low discrepancy; i.e., system of points for which an estimate analogous to \eqref{blattdiscr} holds, especially where the dependence of the constants on dimension are desired to have a polynomial growth with respect to the dimension.
Most of these estimates are in the context of integration of $1$-periodic functions on $[0,1]^q$ with respect to the Lebesgue measure and the total variation is taken in the sense of the so called Hardy-Kraus variation.  
A survey can be found in the book \cite{dick2010digital} of Dick and  Pillichshammer. 
Methods have also been proposed to create a low discrepancy sequence through an accept/reject model for uniform random variables  \cite{dwivedi2019power}.
Existence theorems in the context of general measures and domains are also known in the literature, based mostly on probability theory ideas \cite{tractable, mhaskar2020dimension}.
We note finally that the emphasis here is on approximation of an integral by an \textbf{unweighted average} of the samples of the integrand, not on quadrature formulas where a suitably weighted average of the samples can yield substantially better estimates under various smoothness assumptions on the integrand.

The question of function approximation based on samples of the target function on a graph are well studied, especially in the context of band-limited functions. For example, the papers \cite{anis2016efficient, chen2015discrete} discuss algorithms for obtaining points on a graph with the property that it is possible to reconstruct band-limited functions on the graph exactly using samples at these points. The question is studied from the point of view of compressive sensing in \cite{puy2018random}.
The paper \cite{pesenson2008sampling} presents a detailed study of the space of band-limited functions on a graph and the sets of uniqueness for such functions. 
A key fact that characterizes such sets is that a Marcinkiewicz-Zygmund inequality holds for the space of band-limited functions involved. 
In turn, this leads to quadrature formulas exact to for integration of these spaces. 
In this paper, we deal with signals that are not necessarily band-limited; indeed, the notion of spectral decomposition of the graph Laplacian plays no role in our theory.
On the other hand, our interest is in approximating integrals of functions rather than the functions themselves.

A standard reference on potential theory is the book \cite{landkof1972foundations} by Landkof. 
We are not aware of any prior work specifically for potential theory on graphs. 
However, there are a number of papers dealing with potential theory on locally compact spaces, e.g., \cite{fuglede1960theory, ohtsuka1961potentials,  mhaskar1990weighted}. 
The notion of a sequence of Leja points was introduced in \cite{leja1957certaines} in the context of approximation of the equilibrium measure of a compact subset of the complex plane. 
A detailed analysis of the rate of convergence of the sequence of measures $\sigma_n$ that associates the mass $1/n$ with the first $n$ points in this sequence to the equilibrium measure is given most recently by Pritsker \cite{pritsker2011equidistribution}. 
This notion of discrepancy theorems can be generalized in other, more general contexts. 
A survey can be found in \cite{de2004leja} by De Marche, where computational issues are discussed.
Discrepancy theorems for Leja points in the context of hyper-spheres is analyzed by G\"otz in \cite{gotz2001distribution}.

After the submission of our paper, we came across a paper by Brown \cite{brown2021sequences}, where the author has given a construction of a sequence of good discrepancy points on a graph based on the Green function of a power of the graph Laplacian, where the discrepancy is measured in terms of the Wasserstein metric. 
Naturally, the error in integration is estimated analogously to the Hlawka-Koksma inequality \cite{kuipers2012uniform} in terms of the Lipschitz constant of the function.
In the current paper, we propose a construction independently of the eigen-decomposition of the graph Laplacian, and estimate the discrepancy using the matrix involved in our construction.

There has been recent interest in coreset selection on general domains, including graphs.   A general selection of points that are well distributed on graphs has been found through randomized QR decompositions  \cite{bermanis2013multiscale}, and through random walk sampling \cite{jin2011albatross, nazi2015walk}, however these results do not provide estimates on the error in approximation of an integral.  There has also been significant work on sample selection for full reconstruction of the underlying signal (see \cite{tanaka2020sampling} for a review of methods).
 In \cite{steinerberger2020generalized}, the author provided bounds to guarantee the existence of quadrature formulas (known in the paper as graphical designs), which find exactly the averages of  eigenfunctions of the graph Laplacian corresponding to large eigenvalues.
 In \cite{linderman2020numerical}, the authors provide  bounds for the quadrature error in computing the average of function values based on the values of the function at  arbitrary points {\bf and quadrature weights}. These bounds apply only for spectrally band limited functions. The bounds  depend upon  the spectral band, the $L^2$ norm of the target function, and certain powers of the graph Laplacian applied to the vector of the quadrature weights involved.  Algorithms to choose these points and weights can be found in \cite{lu2020quadrature} for manifolds and \cite{vahidian2020coresets} for a greedy algorithm of point and weight selection on graphs.  

\subsection{Motivating Example}\label{motivatingexample}
As an example, in Figure \ref{fig:qmc_compare} we display a quasi-Monte-Carlo sampling scheme on the unit square and compare it to the various proposed sampling schemes from this paper.  For the quasi-Monte-Carlo sampling, we use a Halton sequence on the unit square, skipping the first 1000 samples and with a leap of 100 \cite{wang2000randomized}. 
For our proposed methods, we begin with 10,000 points uniformly sampled on the unit square $[0,1]^2$, and build a 50-nearest-neighbor graph.  The weighted edges are computed using: a Gaussian weight $A(x_i, x_j) = e^{-\|x_i - x_j\|^2/0.01^2}$, or a log potential weight $A(x_i,x_j)=\log(1/(\|x_i-x_j\|+\varepsilon))$ with $\varepsilon=10^{-5}$.  We select 1000 Leja points using the algorithm and construction proposed in Sections \ref{bhag:lejapts} and \ref{bhag:setup}.  For each graph adjacency type, this is done for two different measures that we wish to integrate against: a uniform distribution on $[0,1]$, or a nonuniform radial measure $\nu^*(x) \propto e^{-\left\|x - c\right\|^2/0.25^2}$ where $c= (0.5 , 0.5)$.


It is clear from this example that the proposed graph Leja points with a uniform distribution recover a similar low discrepancy layout to quasi-Monte-Carlo sampled points.  And beyond this, the graph Leja points are able to generalize to a non-uniform distribution with the same framework.   But the true benefit of this proposed Leja point construction is that the algorithm applies in the case of an arbitrary graph, and will still result in well-spaced points.

\begin{figure}[t]
\centering
\footnotesize
\begin{tabular}{ccc}
\includegraphics[width=.3\textwidth]{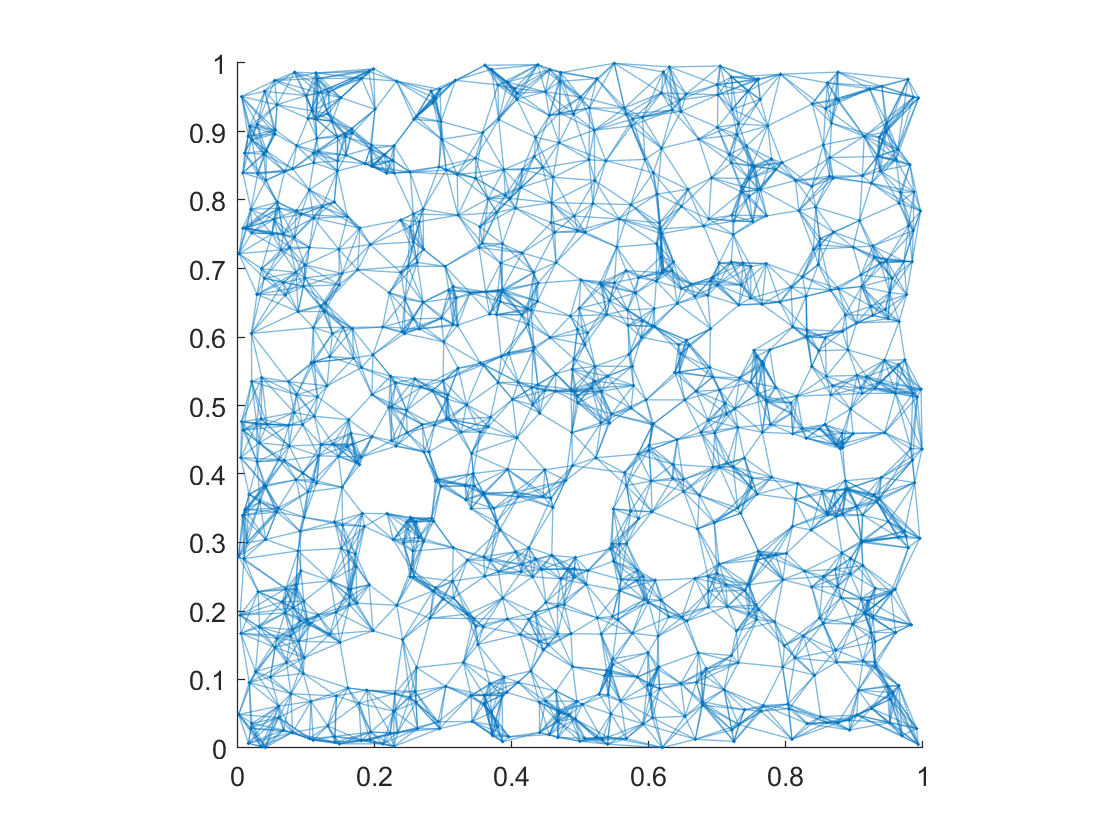} & 
\includegraphics[width=.3\textwidth]{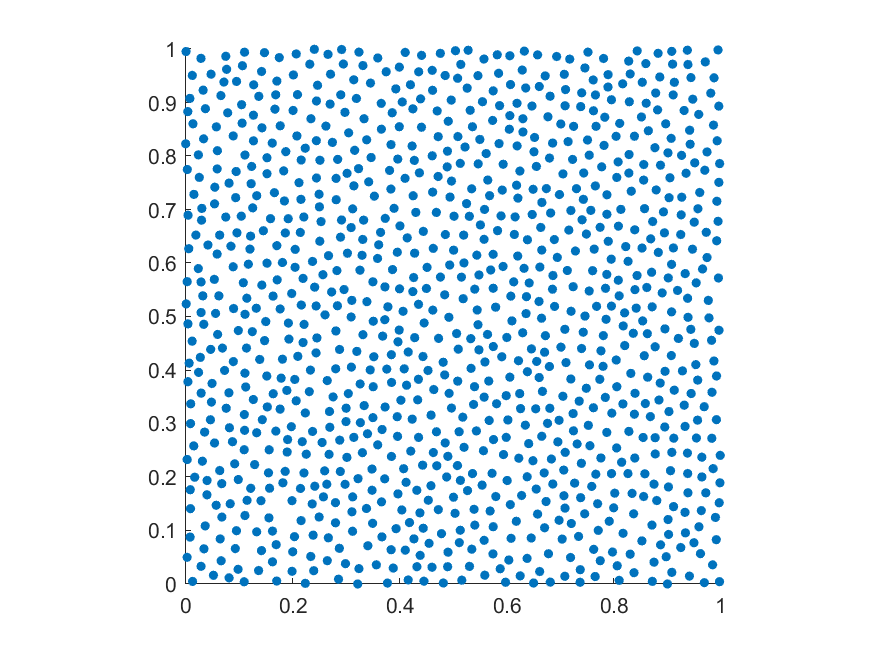} & 
\includegraphics[width=.3\textwidth]{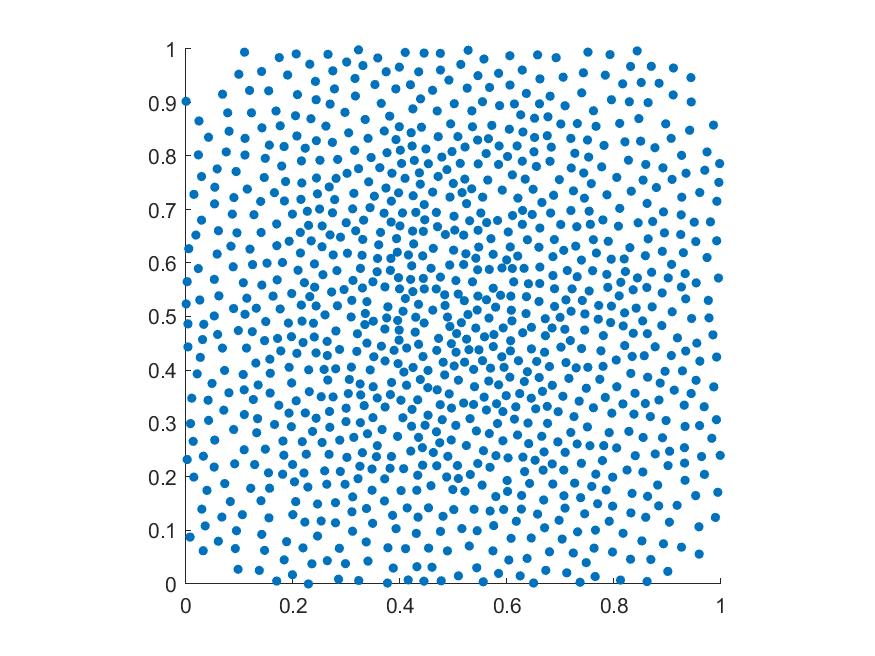} \\
Original graph & Leja Points, Gaussian, Unif. measure &  Leja Points, Gaussian, radial measure
\end{tabular}
\begin{tabular}{ccc}
\includegraphics[width=.3\textwidth]{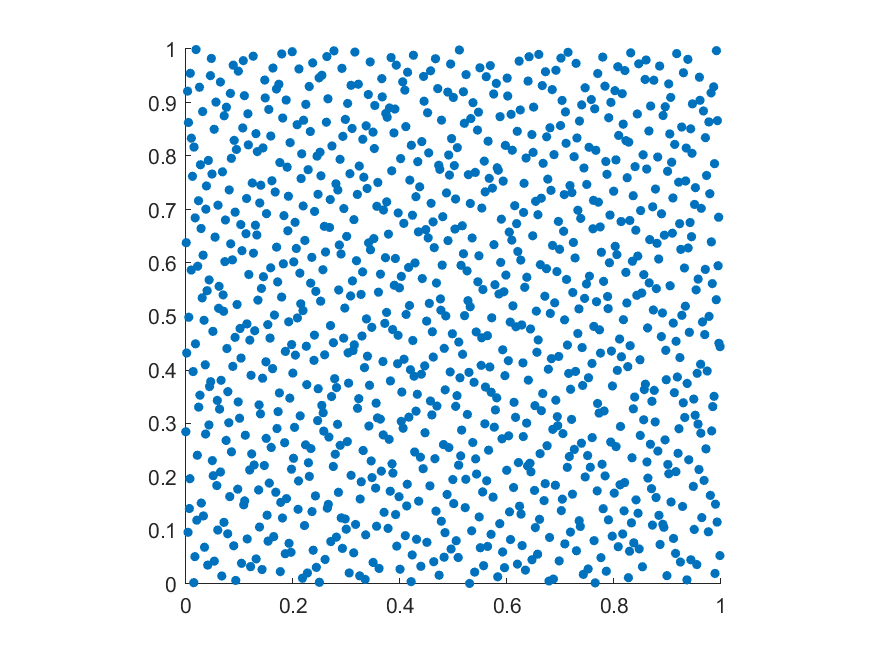} & 
\includegraphics[width=.3\textwidth]{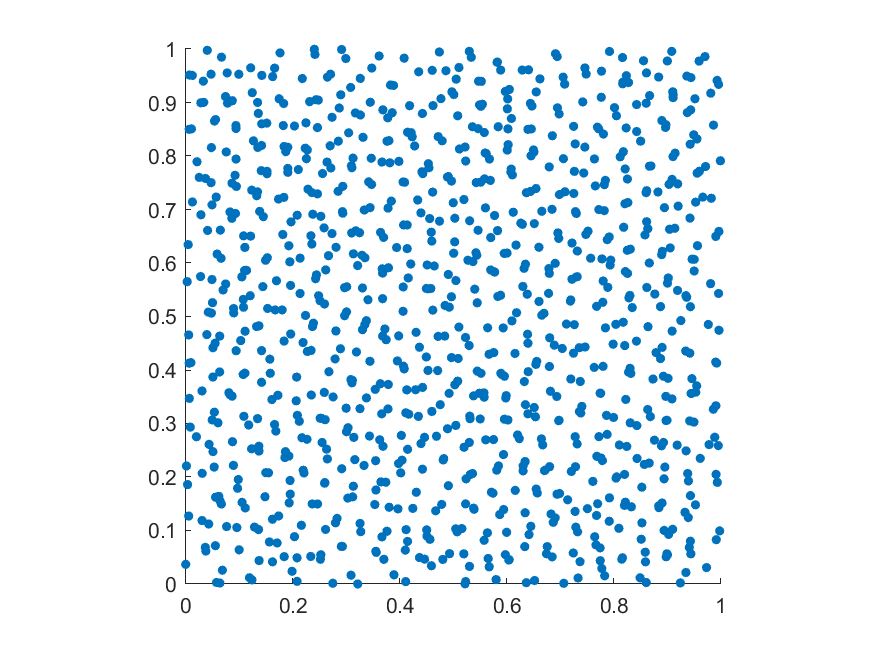} & 
\includegraphics[width=.3\textwidth]{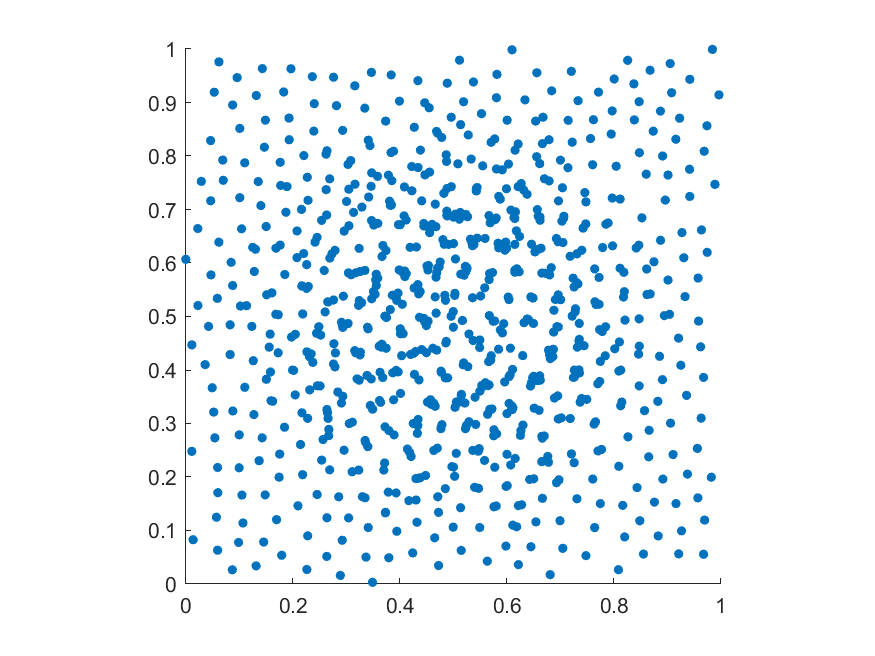}\\
Halton QMC sequence & Leja Points, Potential, Unif. measure & Leja Points, Potential, radial measure
\end{tabular}
\caption{Examples of various sampling schemes on the unit square.  We select 1000 points for sampling.  For Leja points, this is a chosen subset of 10,000 iid uniform distribuiton points.  See Section \ref{motivatingexample} for a full description.  Note, original graph image was computed on smaller number of points/edges for visualization purposes only.}\label{fig:qmc_compare}
\end{figure}


\subsection{Outline of the paper}\label{bhag:outline}
We describe our main theorem in Section~\ref{bhag:main} after reviewing some basic facts from potential theory in the abstract. 
As mentioned earlier, our approach is not to start with a matrix and work with whatever equilibrium measure is associated with it, but rather to start with a graph adjacency matrix and a given probability distribution on the vertices of the graph, and construct another matrix with the same properties as the graph for which this distribution is the equilibrium distribution. 
We describe three such constructions in Section~\ref{bhag:construction}. 
The theory is illustrated with various synthetic and real world examples in Section~\ref{bhag:experiments}. 
The proof of the results in Section~\ref{bhag:main} are given in Section~\ref{bhag:proofs}.

\bhag{Main theorem}\label{bhag:main}
We develop some basic notation in Section~\ref{bhag:notation}, and review a fundamental theorem in potential theory in Section~\ref{bhag:potential}. 
The notion of Leja points is defined in Section~\ref{bhag:lejapts}.  Our main theorem, Theorem~\ref{theo:quadrature} shows that the sequence of Leja points is a low discrepancy sequence for integration with respect to the given measure $\nu^*$.

\subsection{Notation}\label{bhag:notation}
Let $G$ be a $N\times N$ symmetric matrix, $\XX$ be a finite set with $|\XX|=N$. 
We prefer to index $G$ with $\XX\times\XX$.

We consider any function $\nu :\XX\to\RR$ to be a measure on $\XX$ as well as a function on $\XX$, as well as a vector. So, for any function $f:\XX\to\RR$,
\be\label{integralnotation1}
\int fd\nu =\sum_{x\in\XX}f(x)\nu(x)=\int \nu df.
\ee
For measures $\nu, \mu$ on $\XX$, we use the notation
\be\label{integralnotation2}
G(x,\nu)=\int G(x,y)d\nu(y), \qquad G(\mu,\nu)=\int G(x,y)d\mu(x)d\nu(x).
\ee
In particular, $G(x,\nu)$ is the $x$-th component of $G\nu$.
The class of all probability measures on $\XX$ is denoted by $\mathcal{P}$.
The class of all measures $\nu\in\mathcal{P}$ for which $\nu(x)>0$ for all $x\in\XX$ is denoted by $\mathcal{P}_+$.

We denote the vector $(1,\cdots,1)^\intercal$ by $\bs1$. 


A matrix $G$ is conditionally positive semi-definite (c.p.s.d.) if 
$v^\intercal Gv\ge 0$ for all $v$ with $v^\intercal\bs1=0$, and conditionally positive adefinite (c.p.d.) if it is c.p.s.d. and $v^\intercal Gv=0$, $v^\intercal\bs1=0$ together imply that $v=0$.

For a symmetric matrix $A$, we denote
\be\label{rowsumnorm}
M(A) =\max_{x\in\XX} |A(x,x)|, \qquad \tn A\tn =\max_{x\in\XX}\sum_{y\in\XX} |A(x,y)|, \qquad \tn A\tn' =\max_{x\in\XX}\sum_{\substack{y\in\XX\\ y\not=x}} |A(x,y)|.
\ee
For a vector $v$, we denote the $\ell^p$ norm of $v$ by 
$\|v\|_p$, and write $\kappa(v)=\|v\|_\infty/\min_{x\in\XX}|v(x)|$.

\subsection{Potential theory}\label{bhag:potential}
A measure $\nu^*\in\mathcal{P}$ is the equilibrium measure (vector) if
\be\label{capacitydef}
\Gamma(G)=G(\nu^*,\nu^*)=\min_{\nu\in\mathcal{P}} G(\nu,\nu).
\ee
In the context of complex potential theory, the quantity $\Gamma(G)$ is often called Robin's constant. In our context, we refer to the quantity $\Gamma(G)$ as the \emph{capacity of $G$}.

An important characterization of the equilibrium measure is given by the following lemma, known generically as Frostman theorem. 
A proof can be found in almost any book on classical potential theory (e.g., \cite[pp.~136-137]{landkof1972foundations} in the context of potentials on Euclidean spaces and \cite[Theorem~2.1]{ohtsuka1961potentials} in the context of locally compact spaces).
We will reproduce a proof for the sake of completion in Section~\ref{bhag:proofs}.

\begin{lemma}\label{lemma:frostman}
{\rm (a)} 
If $\nu^*$ is any equilibrium measure, and $S^*$ is the support of $\nu^*$, then
\be\label{frostmanlow}
G(x,\nu^*)\ge \Gamma(G), \qquad x\in \XX,
\ee
and
\be\label{frosmaneq}
G(x,\nu^*)=\Gamma(G), \qquad x\in S^*.
\ee
{\rm (b)} Let $G$ be conditionally positive semi-definite.  
If $c\in\RR$, and any measure $\mu\in\mathcal{P}$ with support $S$ satisfies both 
\be\label{tentfrostlow}
G(x,\mu)\ge c, \qquad x\in \XX,
\ee
and
\be\label{tentfrosmaneq}
G(x,\mu)=c, \qquad x\in S.
\ee
 then $c=\Gamma(G)$, and $\mu$ is an equilibrium measure. If $\nu^*$ is any equilibrium measure, then $G(\mu-\nu^*)=0$.
\end{lemma}

\subsection{Leja points}\label{bhag:lejapts}

In this section, we will assume  a symmetric, conditionally positive semi-definite matrix $G$, which has $\nu^*\in \mathcal{P}_+$ as an equilibrium measure.
In particular,
\be\label{frostman}
G(x,\nu^*)=G(\nu^*,\nu^*)=\min_{\nu\in \mathcal{P}_+}G(\nu,\nu)=\min_{\nu\in\mathcal{P}_+}\max_{y\in\mathsf{supp}(\nu)}G(y,\nu)=\Gamma(G). \qquad x\in \XX.
\ee

\begin{definition}\label{def:lejapts}
A sequence $\{a_k\}_{k=0}^\infty$ of points in $\XX$ will be called a \textbf{\textit{Leja sequence}} (with respect to a matrix $G$) if for every $k\ge 1$,
{\rm
\be\label{lejadefrec}
a_k=\argmin_{x\in \XX}\sum_{j=0}^{k-1}G(x,a_j).
\ee
}
The points $a_k$ will be referred to as Leja points.
\end{definition}
\begin{rem}\label{rem:lejadef}
{\rm
The points $a_k$ are in general not distinct. 
In fact, in order for Theorem~\ref{theo:quadrature} to hold, the points will need to be repeated to be commensurate with the measure $\nu^*$. 
Second, a Leja sequence is not uniquely determined by the initial point $a_0$, since the definition does not require the $\argmin$ to be unique.
}
\end{rem}

We will denote the Dirac delta at $a_k$ by $\nu_k$. For $k\ge 1$, let $\sigma_k=(1/k)\sum_{j=0}^{k-1}\nu_k$.
Then
\be\label{lejadefrecbis}
G(\nu_k,\sigma_k)=\min_{x\in\XX}G(x,\sigma_k), \qquad k\in \ZZ_+.
\ee
If $f$ is the range of $G$, we write
\be\label{gderdef}
\mathcal{D}_G(f)=\argmin\{\|w\|_1 : Gw=f\}.
\ee


\begin{uda}\label{uda:compressive}
{\rm 
Let $G$ be any positive definite, $N\times N$ matrix,  $\mathcal{F}$ be the class of all convex combinations of  columns of $G$. 
If $f\in \mathcal{F}$, then $\|\mathcal{D}_G(f)\|_1=1$.
\qed}
\end{uda}

\begin{uda}\label{uda:modifiedgreen}
{\rm
Let $A$ be the weighted adjacency matrix of a connected graph that does not have non-trivial, bipartite connected subgraph, and $\mathcal{L}$ be the graph Laplacian, and  
$\{\lambda_k\}$ be the set of eigenvalues of $\mathcal{G}$, and $\{\phi_k\}$ be the corresponding orthonormalized eigenvectors.  
Then (\cite[Lemma~1.7]{chung1997spectral}) 
$\{\lambda_k\}\subset [0,2)$.
We let $G=I-(1/2)\mathcal{L}$. 
Then $G$ is positive definite with non-negative entries, and positive diagonal, so that Theorem~\ref{theo:quadrature} is applicable, and the equilibrium measure for $G$ is defined uniquely.
In particular, $G$ is invertible.
If $f : V\to\RR$, $f=\sum_k \hat{f}(k)\phi_k$, then
$$
\mathcal{D}_G(f)=2\sum_k \frac{\hat{f}(k)}{2-\lambda_k}\phi_k.
$$
Let $\lambda<2$ and $\Pi_\lambda$ be the class of all $\lambda$-band-limited functions; i.e., the class of all $f :\XX\to\RR$ for which $\hat{f}(k)=0$ if $\lambda_k\ge \lambda$.
For $f\in \Pi_\lambda$, a crude estimate for $\|\mathcal{D}_G(f)\|_1$ is given by
$$
\|\mathcal{D}_G(f)\|_1\le \frac{2}{2-\lambda}\max_{k :\lambda_k<\lambda}\|\phi_k\|_1\sum_{k: \lambda_k<\lambda}|\hat{f}(k)|.
$$
 \qed}
\end{uda}

Our main theorem is the following.

\begin{theorem}\label{theo:quadrature}
Let  $G(x,y)\ge 0$ for all $x,y\in\XX$, $\{a_k\}$ be a Leja sequence with respect to $G$, $\nu^*\in \mathcal{P}_+$ be an equilibrium measure for $G$, and $f$ be in the range of $G$. 
Then
\be\label{quadest}
\left|\int_\XX fd\nu^*-\frac{1}{n}\sum_{k=0}^{n-1}f(a_k)\right|\le \frac{3\tn G\tn}{n+1}\left\|\mathcal{D}_G(f)\right\|_1.
\ee
\end{theorem}

\begin{rem}\label{rem:bandlimited}
{\rm
In contrast to the estimates given in \cite{linderman2020numerical, vahidian2020coresets}, our estimates do not require the function $f$ to be band-limited. 
\qed}
\end{rem}
\begin{rem}\label{rem:quadrem}
{\rm
We note that the rate of  convergence of the equal weight quadrature formulas is much faster than what is expected for quasi-Monte Carlo methods, demonstrating that the sequence of Leja points is a low discrepancy sequence in the sense of \cite{dick2010digital}.
It is not clear what conditions on the matrix $G$ will ensure an equilibrium measure in $\mathcal{P}_+$.
In Section~\ref{bhag:construction}, we will give a variety of possible constructions to modify either an arbitrary matrix or a matrix with certain properties to ensure that any given vector $\nu^*\in \mathcal{P}_+$ is an equilibrium measure of the resulting matrix.
Thus, rather than taking the viewpoint that $G$ is given and $\nu^*$ is its equilibrium measure, we will start with $\nu^*$ and construct $G$ with $\nu^*$ as the equilibrium measure. 
Then the Leja sequence depends upon $\nu^*$ via $G$.
\qed
}
\end{rem}

The proof of Theorem~\ref{theo:quadrature} mimics the standard proof of the convergence of the measures $\nu_n$ in classical potential theory. 
Thus, we note that the mapping $(\mu,\nu)\to G(\mu,\nu)$ is a semi-inner product on the space of all measures on $\XX$. 
We prove first that the sequence $\nu_n$ converges to $\nu^*$ in the sense of the semi-norm defined by this semi-inner product. 
This implies in turn that  the Ces\'aro means $(C,2)$ of the sequence $\nu_n$ converges to $\nu^*$. 
A tauberian argument then completes the proof.
Of course, we need to keep track of the rates of convergence at each stage.

\bhag{Constructions}\label{bhag:construction}

Let $B$ be any symmetric matrix, and $v\in\mathcal{P}_+$.
We want to construct a matrix $G$ such that $G$ is conditionally positive semi-definite,  and $Gv=c\bs1$, so that $v$ is the equilibrium vector for $G$.
We describe three constructions.

\subsection{Diagonal modification}\label{bhag:diagmod}
This construction gives a modified graph Laplacian, and works with any symmetric matrix $B$ with no further assumptions. First, we construct a symmetric matrix $B_1$ so that $B_1v=0$.  \\

We define the \textbf{\textit{$v$-Laplacian}} $L_v(B)$ by
\be\label{straightb1}
w(x)=\frac{(Bv)(x)}{v(x)}, \qquad W=\mathsf{diag}(w(x)), \qquad L_v(B)=W-B.
\ee
Then with $B_1=L_v(B)$,
\be\label{straightnormest}
 \tn B_1\tn \le (\kappa(v)+1)\tn B\tn.
\ee

Let $V=\mathsf{diag}(v(x)/\|v\|_\infty)$, and 
\be\label{gdef}
G=2\tn B_1\tn V^{-1}-B_1.
\ee 
Then
$$
G(x,x)-\sum_{y\not=x}|G(x,y)| \ge 2\|v\|_\infty\tn B_1\tn/v(x) -\tn B_1\tn \ge \tn B_1\tn >0.
$$
So, $G$ is diagonal dominant, hence, positive definite, and
\be\label{gequil}
Gv=2\|v\|_\infty\tn B_1\tn \bs1^\intercal,
\ee 
so that $v$ is the unique equilibrium measure of $G$ and 
\be\label{gcapacity}
\Gamma(G)=2\|v\|_\infty\tn B_1\tn.
\ee
Moreover,
\be\label{b2normest}
\tn G\tn/2\le \max_{x,y\in \XX}|G(x,y)|\le \tn G\tn \le (2\kappa(v)+1)\tn B_1\tn.
\ee

\begin{rem}\label{rem:vlaplacian}
{\rm
The construction of the $v$-Laplacian works for \emph{all} symmetric matrices $B$ with no further assumptions. 
In the case when $B$ is the (weighted) adjacency matrix with non-negative weights, then with $V=\mathsf{diag}(v(x))$, $VBV$ is another adjacency matrix for the same graph so that $B(x,y)=0$ if and only if $(VBV)(x,y)=0$.
The matrix $L(VBV)=VL_v(B)V$ is the non-normalized graph Laplacian for $VBV$. 
Thus, the eigenvalues of $L_v$ are the same as those of $L(VBV)$, and if $\phi$ is an eigenvector of $L(VBV)$, then the eigenvector for $L_v$ for the same eigenvalue is $V\phi$.
In particular, $L_v(B)$ is positive semi-definite, and the graph is connected if and only if $v$ is the unique null vector for $W-B$.

If $B$ is the weighted adjacency matrix of a connected graph (with positive weights) and $B_1$ is the $v$-Laplacian defined in \eref{straightb1}, the matrix $G$, defined by \eqref{gdef},  has the same graph structure as $B$ except for self-loops, and $G(x,y)\ge 0$ for all $x,y\in\XX$. 
\qed
}
\end{rem}

\subsection{Householder transform}\label{bhag:householder}

In the case when $B$ is not invertible, there is another way to construct $B_1$, which essentially preserves the eigenstructure of $B$ itself (rather than a graph Laplacian for $B$). 
This construction is also applicable for every non-invertible $B$ which may have negative entries.
If $u$ and $w$ are unit vectors, the Householder transform is defined by
\be\label{householder}
H[w,u]=I-2\frac{(w-u)(w-u)^\intercal}{\|w-u\|_2^2}.
\ee
Clearly, $H[w,u]$ is a symmetric unitary matrix, $H[w,u](w)=u$, and $H[w,u]=H[u,w]$, so that $H[w,u](u)=w$. 
We have
\be\label{hausholdernorm}
\tn H[w,u]\tn \le 1+2\frac{\|w-u\|_\infty\|w-u\|_1}{\|w-u\|_2^2}\le 1+2\kappa(w-u).
\ee

Since $B$ is not invertible,  there is a unit vector $u$ such that $Bu=0$. 
Writing $\tilde{v}$ for the unit vector along $v$, we set 
\be\label{houseb1}
B_1=H[u,\tilde{v}]BH[u,\tilde{v}].
\ee
 Clearly, $B_1v=0$,
\be\label{housenormest}
 \tn B_1\tn \le \tn H[u,\tilde{v}]\tn^2\tn B\tn.
\ee
The matrix $B_1$ has the same eigenvalues as $B$, and if $B=U\Lambda U^\intercal$ is the spectral decomposition of $B$, then the eigenvectors of $B$ are obtained simply by replacing every eigenvector $u_j$ of $B$ by $Hu_j$. 
However, this construction depends upon $B$ having a null vector $u$ and our ability of compute $H[u,\tilde{v}]$, and the norms will depend upon the norms of this Householder transform. 

With the matrix $B_1$ defined in \eqref{houseb1}, we construct $G$ as in \eqref{gdef}.

\subsection{Symmetric scaling}\label{bhag:slinkhorn}
This construction applies only to symmetric, \textbf{non-negative} matrices $B$, and results in a matrix with the same eigenvalues as those of $B$.
For a matrix $B$, and subsets $S_1,S_2\subseteq\XX$, we denote by $B[S_1,S_2]$ the sub-matrix of $B$ obtained by extracting rows of $A$ indexed by $S_1$ and columns of $B$ indexed by $S_2$. 

Let $B$ be a non-negative, symmetric matrix, $v\in\mathcal{P}_+$ satisfy the Brualdi condition\\
\emph{For all partitions $S_1, S_2, S_3$ of $\XX$ such that $B[S_2\cup S_3, S_3]=0$, we have
\be\label{brualdicond}
\sum_{x\in S_1}v(x) \ge \sum_{y\in S_3}v(y),
\ee 
with equality holding if and only if $A[S_1,S_1\cup S_2]=0$.
}\\

Then a theorem of Brualdi \cite{brualdi1974dad} states that there exists a diagonal matrix $D$ with positive entries such that
\be\label{brualdi_scaling}
DBD\bs 1=v.
\ee 
Let $V$ be the diagonal matrix $(v(x))$. 
Then $G=V^{-1}DBDV^{-1}$ satisfies
\be\label{brualdi_equilibrium}
Gv=\bs 1.
\ee
Clearly, $G$ represents a graph with the same vertices and edges as $B$, except for different edge weights.
We note that when $B$ has a positive diagonal, then the Brualdi condition is satisfied vacuously for every $v\in\mathcal{P}_+$.  The full algorithmic steps for finding such a bistochastic normalization of symmetric  matrices can be found in \cite{marshall2019manifold, landa2021doubly}.

\bhag{Experiments}\label{bhag:experiments}

\subsection{Set up}\label{bhag:setup}
One matrix $B$ that is trivially row-scalable as in \eqref{brualdi_scaling} is the matrix of all positive entries; i.e., the graph is complete with self-loops at each vertex.
Considering a complete matrix prevents a possible blow-up of $D$ on singleton vertices, where otherwise a singleton vertex $x$ would have $D(x,x) = (B(x,x) v(x))^{-1/2}$.
On the other hand, it is more expensive to apply the Sinkhorn algorithm \cite{sinkhorn1967concerning} to a dense matrix $B$   than applying the algorithm to a sparse matrix $B$, since the most expensive part of the algorithm is the matrix vector multiplications.  

A sparse matrix $B$ can be augmented to have all positive entries in a style similar to Pagerank \cite{page1999pagerank} by adding a small weighted edge between any pair of vertices,
\be\label{pagerank_construction}
\widetilde{B} = (1-\alpha) B + \frac{\alpha}{N} \bs{1}\bs{1}^\intercal.
\ee
Then $\widetilde{B}$ trivially satisfies the assumptions in Section~\ref{bhag:slinkhorn} since it has a positive diagonal.
Moreover, because of the particular structure of $\widetilde{B}$, it is possible to compute a matrix vector multiplication in a time that depends only on the sparsity of $B$.  
This is done using the trivial observation that
\be\label{fastpagerank}
\widetilde{B}v = (1-\alpha) Bv + \frac{\alpha}{N} \bs{1} (\bs{1}^\intercal v )= (1-\alpha) Bv + \frac{\alpha\cdot \sum_x v(x)}{N} \bs{1}. 
\ee
This observation of the Pagerank modification was originally made in \cite{knight2008sinkhorn}.

Similarly,  the matrix $G$ used in \eqref{lejadefrec} can be constructed from the sparse matrix $B$ by observing that
\be\label{sparseG}
G=V^{-1}D\widetilde{B}DV^{-1}= (1-\alpha)V^{-1}DBDV^{-1}  + \frac{\alpha}{N} (D {\nu^*}^{-1} ) (D {\nu^*}^{-1})^\intercal,
\ee
where $V=\mathsf{diag}(\nu^*(x))$.

The entire algorithm for constructing the Leja points  $\{a_j\}$ of the Pagerank graph in Eq. \eqref{pagerank_construction} is shown in Algorithm \ref{algo:leja} for completeness.   
The computational complexity of the algorithm is mostly rooted in the Sinkhorn iterations for computing $G$.  Let $B$ be an $N\times N$ matrix with $k$ edges per node (the number of edges need not be fixed, but simplifies the complexity calculation).  
Then each matvec operation $Bv$ requires $O(Nk)$ flops, so computation of $\widetilde{B}v$ also requires $O(Nk)$ flops.  
Exact bounds on the number$L$  of Sinkhorn iterationsis not known, but we can still denote the computational complexity as $O(NkL)$.  Computing $G$ similarly requires $O(Nk)$ flops and can be held in memory using sparse storage using \eqref{sparseG}.  Finally, computing the Leja points requires summing up to $n$ columns of $G$, which has computational complexity $O(nk)$.  This gives a total complexity of $O(NkL) + O(nk)$. 

\begin{algorithm}
    \caption{Graph Leja Point Selection}
    \label{algo:leja}
    \hspace*{\algorithmicindent} \textbf{Input:} Sparse matrix $B$, Pagerank factor $\alpha$, Base measure $v$, Number Leja points $n$, Queryable function $f$ \\
    \hspace*{\algorithmicindent} \textbf{Output:} Integral estimate $\widehat{\mathbb{E}_\nu}[f]$
    \begin{algorithmic}[1]
    \State $\widetilde{B} = (1-\alpha)B + \frac{\alpha}{N} \bs{1}\bs{1}^\intercal$ \Comment{Notational only, the full ones matrix need not be stored in memory (see \eqref{fastpagerank})}
    \State Sinkhorn iterations to find $D$ such that $D\widetilde{B} D \bs{1} = v$
    \State $G = V^{-1}D\widetilde{B}D V^{-1}$
    \State $a_0 = \textnormal{random vertex } x\in \XX$
    \For{$k$ from $1$ to $n-1$}
    	\State $a_k=\argmin\limits_{x\in \XX}\sum_{j=0}^{k-1}G(x,a_j)$
    \EndFor
    \State $\widehat{\mathbb{E}_\nu}[f] = \frac{1}{n}\sum_{k=0}^{n-1}f(a_k)$
    \end{algorithmic}
\end{algorithm}

In all the experiments below, we compare three quantities:
\begin{itemize}
\item the average $\frac{1}{n}\sum_{k=0}^{n-1}f(x_k)$, where $x_k$ is sampled randomly (with replacement) according to a weighted probability $\nu^*(x_k)$.  This is denoted ``Random'' and with blue curves in the experiments.
\item the average $\frac{1}{n}\sum_{k=0}^{n-1}f(a_k)$, where $a_k$ are the Leja points of the graph with equilibrium measure $\nu^*$.  This is denoted ``Leja'' and with red curves in the experiments.
\item the average $\frac{2}{n(n-1)}\sum_{k=0}^{n-1}(n-k)f(a_k)$ (cf. \eref{potentialbd}), where the $a_k$ are the Leja points of the graph with equilibrium measure $\nu^*$.  This quantity is considered as a natural weighting of the Leja points that respects the order in which they are drawn.  This is denoted ``Summability'' and with orange curves in the experiments.

\end{itemize}

We also note for all examples to follow in this section, that because both the Leja point selection algorithm and the Monte Carlo sampling are with replacement, it is possible to compute $n>N$ iterations without selecting all $N$ vertices. 
 This additional sampling can be seen as establishing weights, as a point sampled twice is given twice the weight as a point sampled once.  
Additionally, we will use the Pagerank kernel $G$ as in \eqref{pagerank_construction} with $\alpha=0.05$.
 
 Finally, unless otherwise noted,we set $\nu^*$ to be the inverse of the kernel density estimate for the graph,
\be
\nu^*(x) \propto \frac{1}{\sum_y A(x,y)},
\ee
normalized so that $\|\nu^*\|_1 = 1$.    This is an arbitrary choice of $\nu^*$, as the only required condition is for $\nu^*(x)>0$, and is meant to demonstrate that the results apply to more than $\nu^* = \frac{1}{N}\mathbf{1}$. 
This equilibrium measure is of particular interest because, in the point cloud kNN graph setting, it approximately cancels out the sampling density of the points themselves.  More specifically, if the points are sampled from the density $p:\mathbb{R}^d\rightarrow \mathbb{R}_+$ that has comptact supp$(p)$, then 
\be
\mathbb{E}_{\nu^*}[f] = \int_{\textnormal{supp}(p)} f(x) p(x) d\nu^*(x) \approx c_p \int_{\textnormal{supp}(p)} f(x) dx,
\ee
which would be independent of the relative heights of the sampling density.   However, any positive density $\nu^*$ satisfies the necessary assumptions for the proven approximation rates.

\subsection{Synthetic Graphs with Smooth Functions}

A common model for social networks is a so-called \emph{small world network} \cite{watts1998collective}.  Small world networks are graphs with a small number of edges per vertex, but where any pair of vertices is likely to have a small graph distance.  Mathematically, if the average graph distance between two vertices is $d$, a small-world graph with $N$ vertices roughly satisfies $d\propto \log(N)$.  These networks are antithetical to \emph{nearest neighbor graphs} that are generated from a point cloud in some metric space, and have a large network diameter.

\begin{figure}
\centering
\footnotesize
\begin{tabular}{cc}
\includegraphics[width=.4\textwidth]{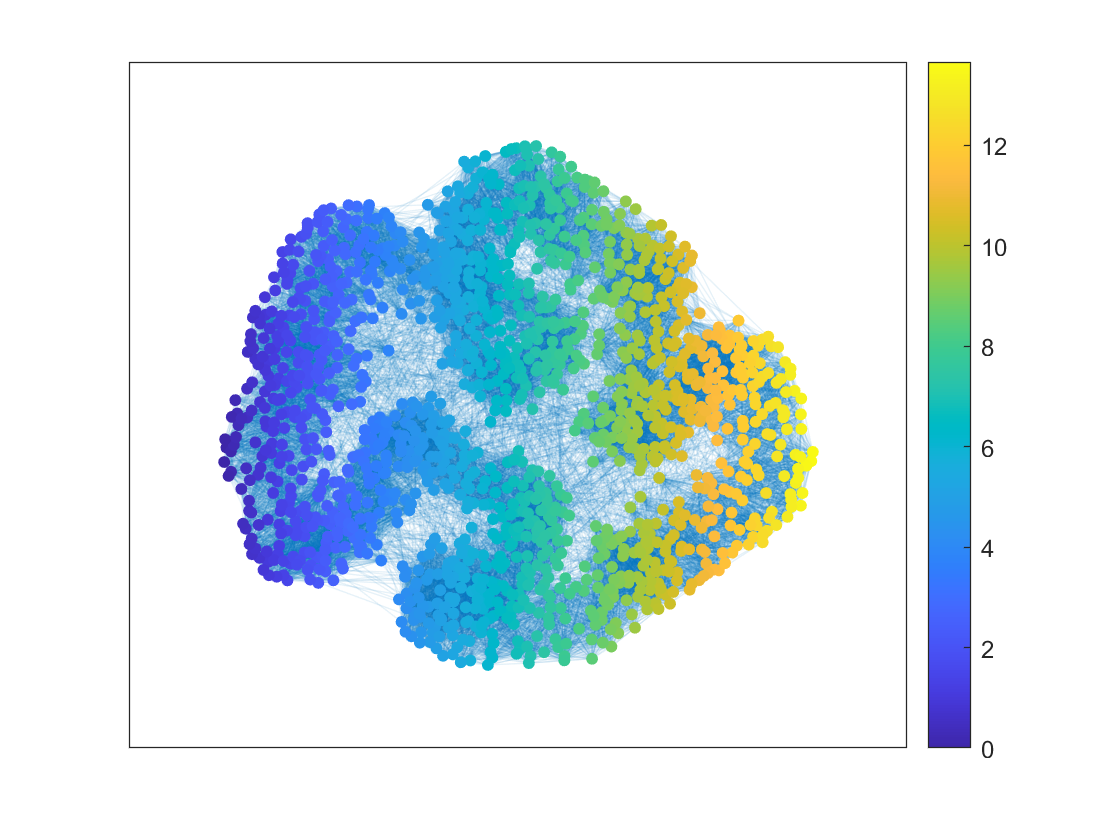} & 
\includegraphics[width=.4\textwidth]{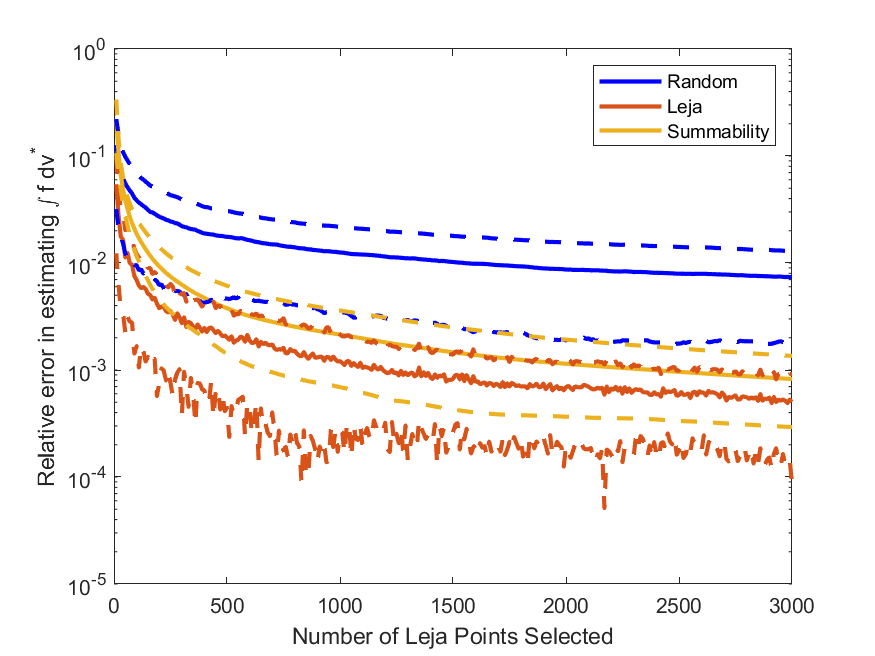} \\
(a) & (b) \\
\includegraphics[width=.4\textwidth]{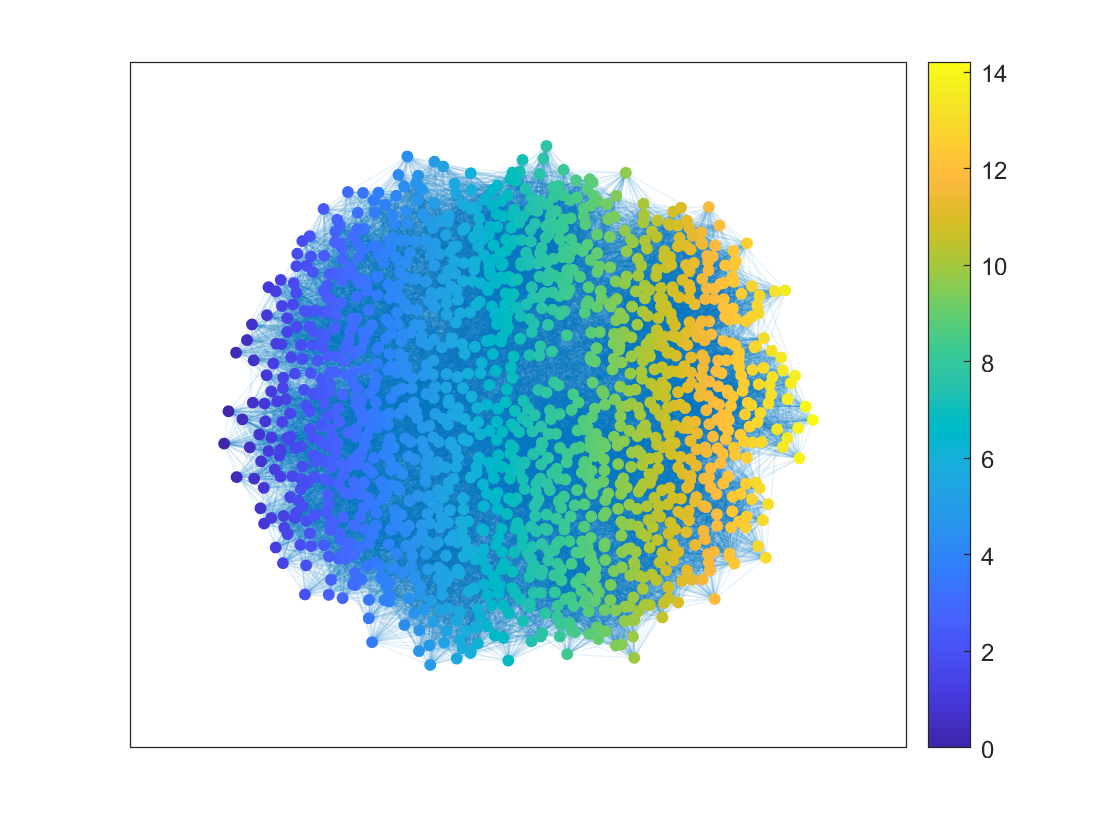} & 
\includegraphics[width=.4\textwidth]{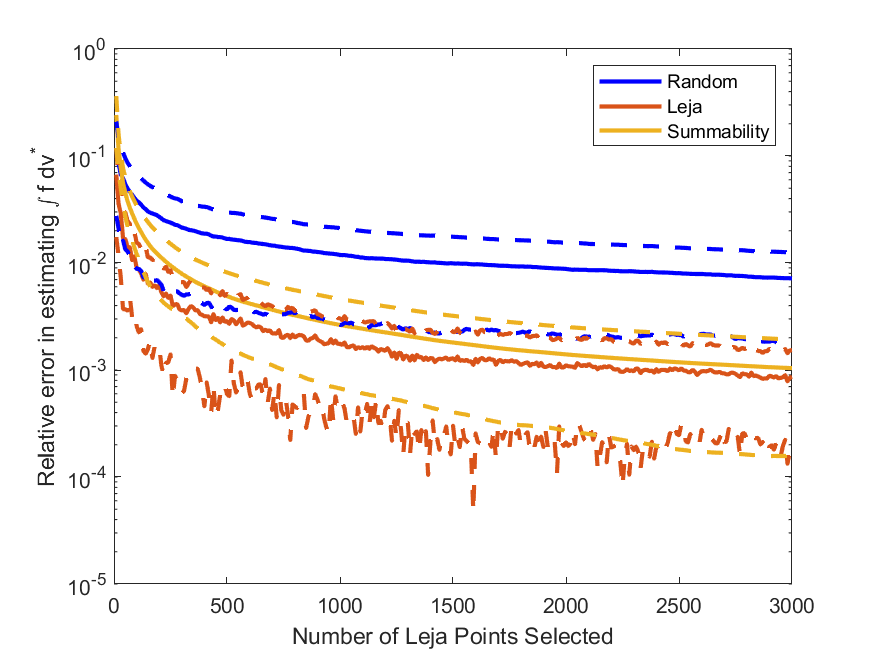} \\
 (c)  & (d)
\end{tabular}
\caption{Watts-Stogatz models for (a)-(b) $\beta=.05$, and (c)-(d) $\beta=.25$.
We display both an example of the graph/function, and the quadrature error results over 1000 instantiations of the random graph.  The dotted lines correspond to the confidence intervals around each mean curve of the same color.}\label{fig:watts}
\end{figure}

In the first set of experiments, we examine the quadrature approximation for a smooth function on a Watts-Strogatz model \cite{watts1998collective}. 
This is a model of random graph generation that exhibits small-world properties, including low path distances between vertices while still exhibiting a high degree of clustering.  The model takes inputs of number of vertices, average number of edges per vertex, and a parameter $\beta\in [0,1]$ that models the fraction of random connections.  When $\beta=0$ the generated graphs will exhibit lattice structure, and when $\beta=1$ the graph will resemble an Erd\H os-R\'enyi random graph.

In this experiment, we take $f$ to be the $x-$coordinate of the vertex locations after applying a force layout \cite{fructerman1991force}.  This results in a smooth function in the nodes, since a force layout is a heuristic algorithm that repels unconnected nodes away from one another.  The functions are displayed in Figure \ref{fig:watts}, and the graphs are constructed for various values of the parameter $\beta$, which measures the level of structure in the graph.  Also, because the graph can be regenerated, we run the experiment for 1000 graphs with the same parameters.  We display the mean and standard deviation across these 1000 instantiations of the graph.  For each graph, we take $N=1500$ vertices, and an average of 25 edges per node.


In a second set of experiments, we construct a nearest neighbor graph from the point cloud in Figure \ref{fig:twocluster}.  
The interesting aspect of this data set is that the points from the two clusters are sampled in a highly non-uniform manner.  
Because $\nu^*$ is constructed by approximating the inverse of the local density, this results in a similar number of points sampled from both clusters, see Figure \ref{fig:leja_synthetic}.    
The function is taken to  be the $x$-coordinate of the points.  Again, because the graph can be regenerated, we run the experiment for 1000 graphs with the same parameters.  We display the mean and standard deviation across these 1000 instantiations of the graph.  For each graph, we take $N=1500$ nodes and use a bandwidth for the Gaussian kernel of $\sigma=.1$.

\begin{figure}
\centering
\begin{tabular}{cc}
\includegraphics[width=.4\textwidth]{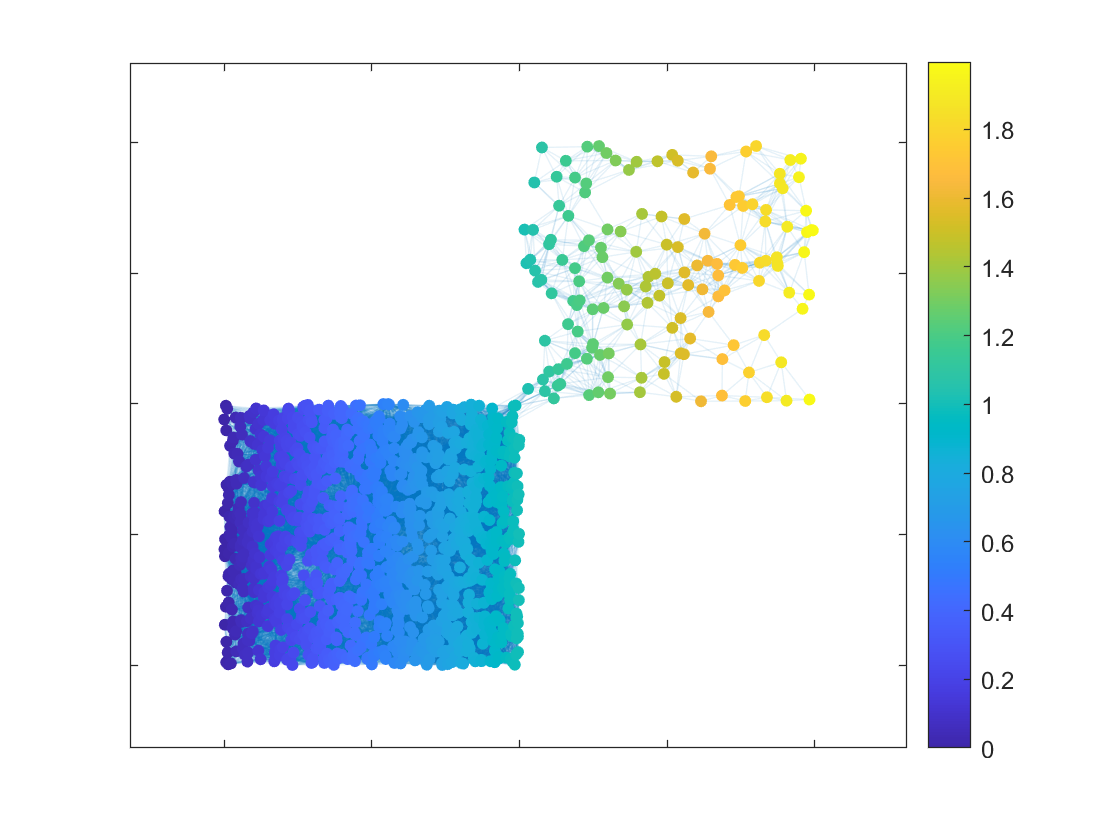} & 
\includegraphics[width=.4\textwidth]{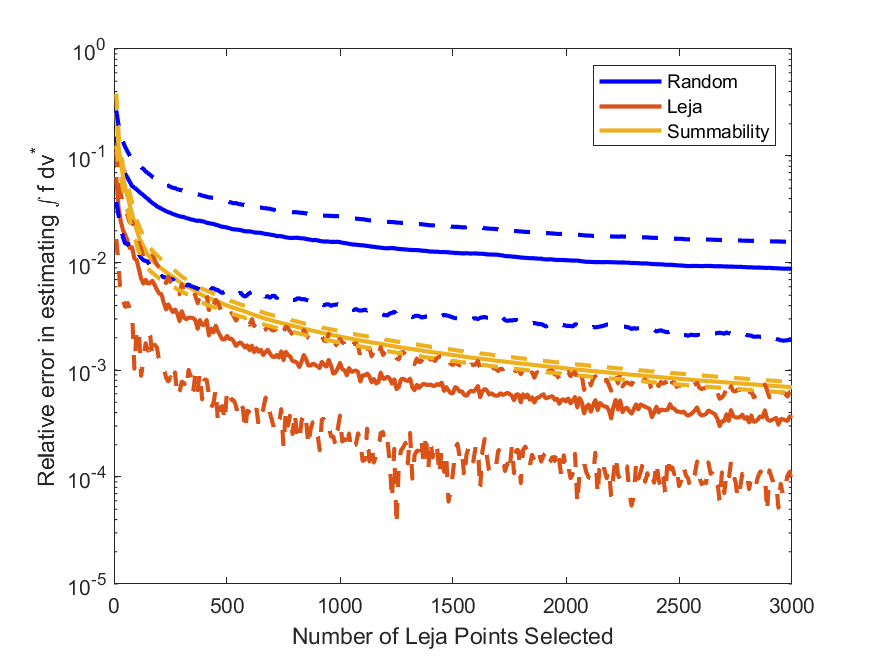} 
\end{tabular}
\caption{Two cluster metric graph example with non-uniform density.  We display both an example of the graph/function, and the quadrature error results over 1000 instantiations of the random graph. The dotted lines correspond to the confidence intervals around each mean curve of the same color.}\label{fig:twocluster}
\end{figure}

We also demonstrate that the location of the Leja points for each of these data sets, using $\nu^*$ to be an inverse density estimate, in Figure \ref{fig:leja_synthetic}.
\begin{figure}
\centering
\begin{tabular}{ccc}
\includegraphics[width=.3\textwidth]{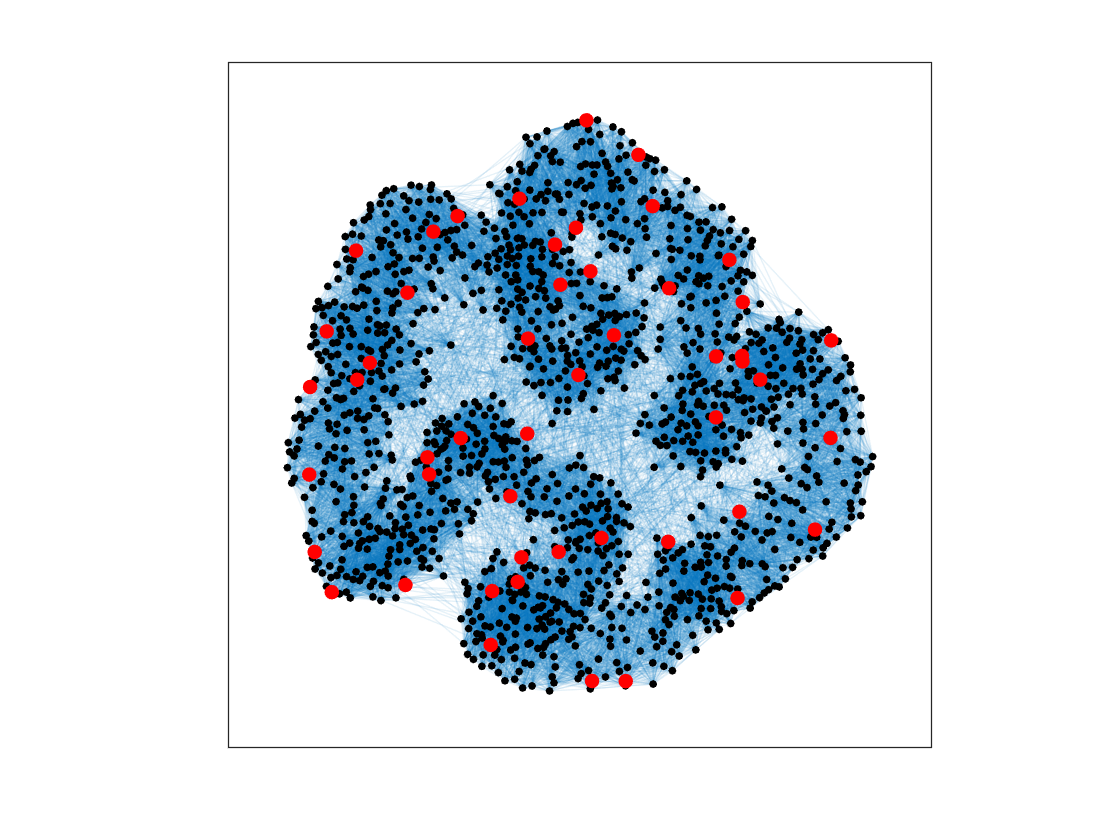} & 
\includegraphics[width=.3\textwidth]{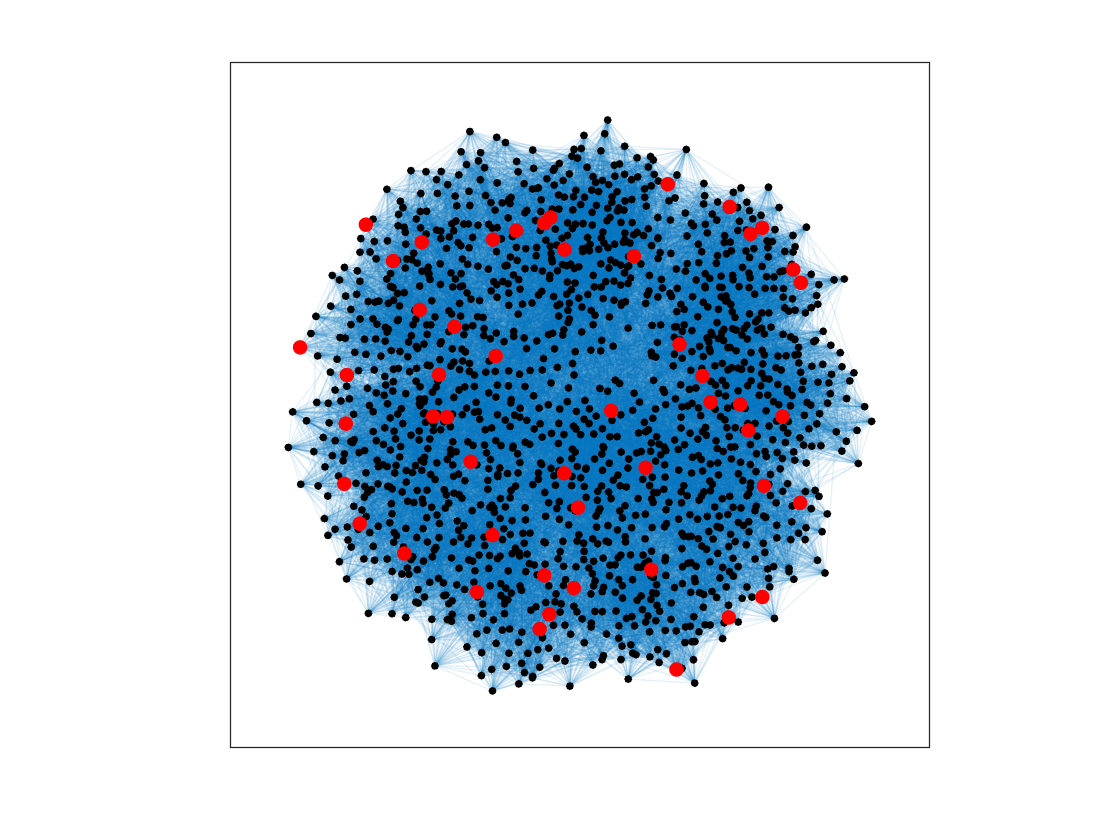} &
\includegraphics[width=.3\textwidth]{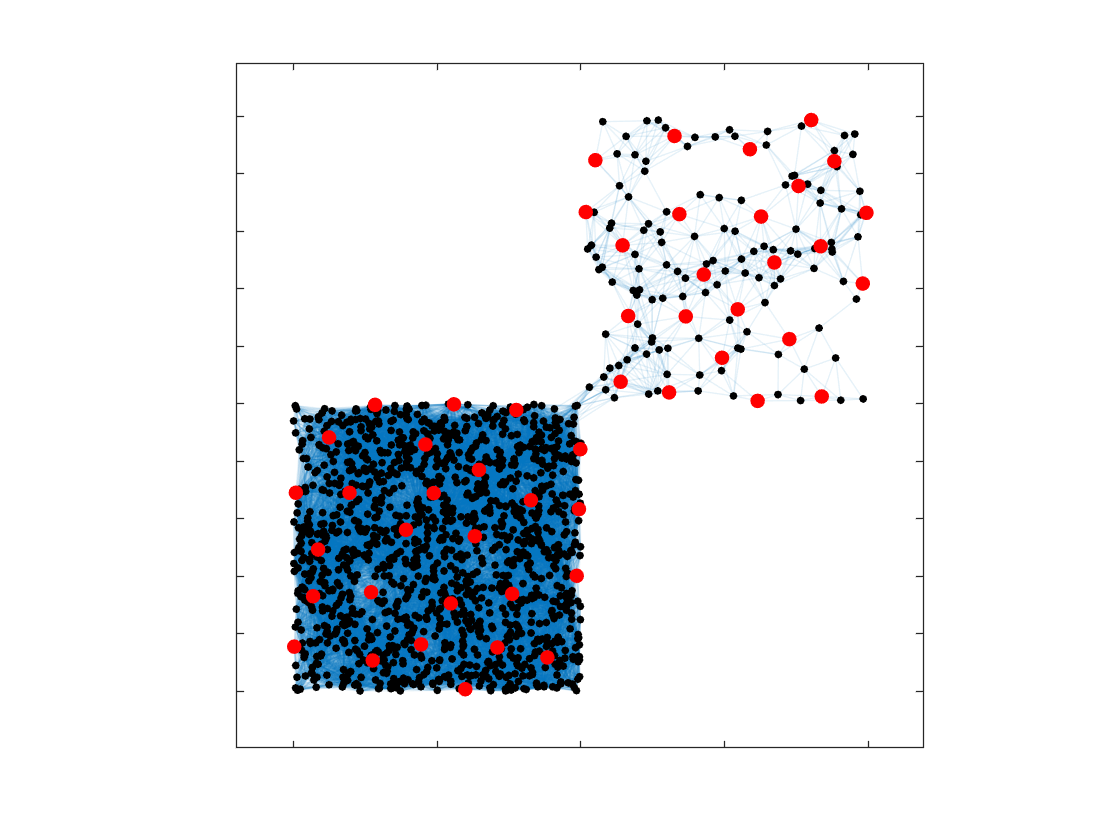} \\
WS graph $\beta=.05$ & WS graph $\beta=.25$ 
& Two Cluster metric graph
\end{tabular}
\caption{Layout of 50 Leja points for various Watts-Stogatz data sets, with $\nu^*$ being the inverse of the kernel density estimate.}\label{fig:leja_synthetic}
\end{figure}

\subsection{Synthetic Graphs with Non-Band-Limited Functions}

A novelty of the results in Theorem \ref{theo:quadrature} is that it applies to fuctions that are not required to be band limited.

We demonsrate the strengths of the bound in Eq. \eqref{quadest}, namely that the function is not required to be band limited in order for the estimation bounds to apply.  To characterize this, we generate a series of experiments on the  
two cluster data from Figure \ref{fig:twocluster}, but with a significantly more complicated function $f$.  To construct $f$, we consider a spectral decomposition of the matrix $G = \Phi \Lambda \Phi^*$.   Then we construct the function spectrally via
\be\label{eq:spectraldecay}
f =  \sum_{k=1}^N\left(\xi_k \lambda_k e^{-\tau k/N}\right) \phi_k,
\ee
where $\xi_k\sim \textnormal{Unif}([0,1])$, and $\tau$ controls the rate of spectral decay with respect to the eigenvalue $\lambda_k$.    In Figure \ref{fig:f_tail}, we demonstrate the quadrature approximation error averaged across $1000$ instantiations of the random graph and random function.  We also consider these experiments across varying $\tau$, demonstrating empirically the effect of the spectral decay rate on the overall quadrature approximation rate.   
\begin{figure}
\centering
\begin{tabular}{cc}
\includegraphics[width=.4\textwidth]{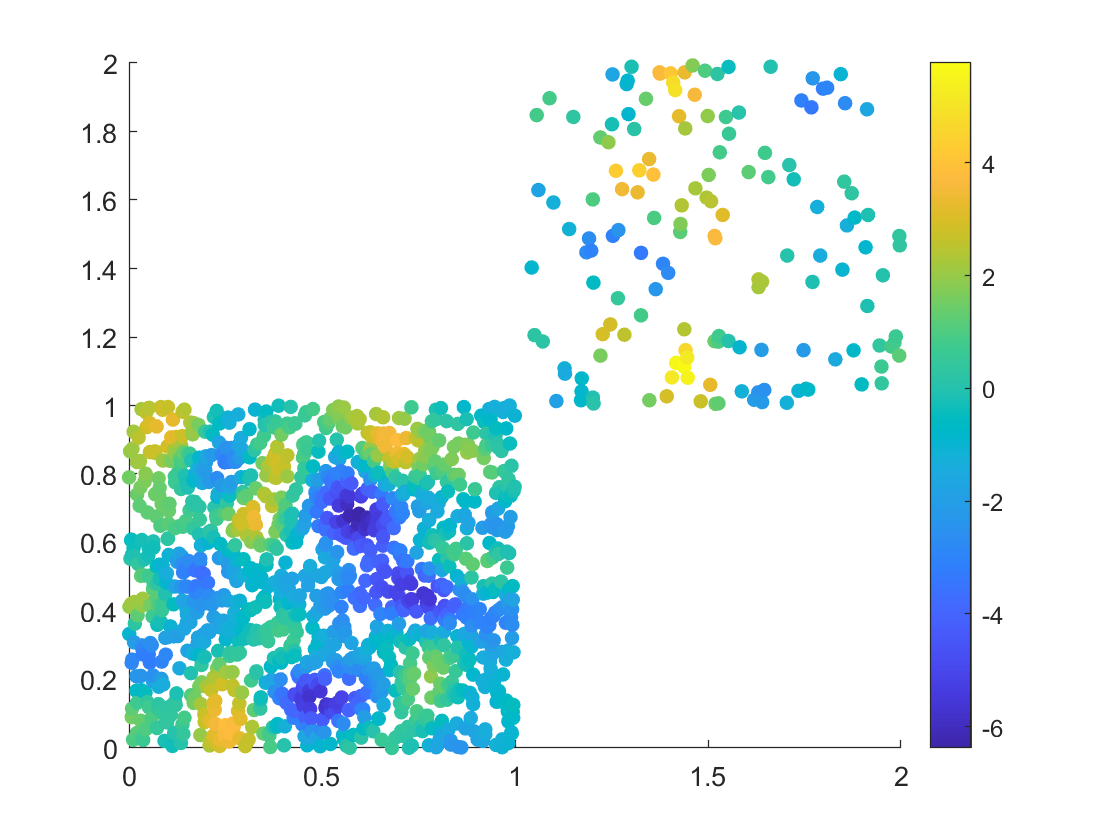} & 
\includegraphics[width=.4\textwidth]{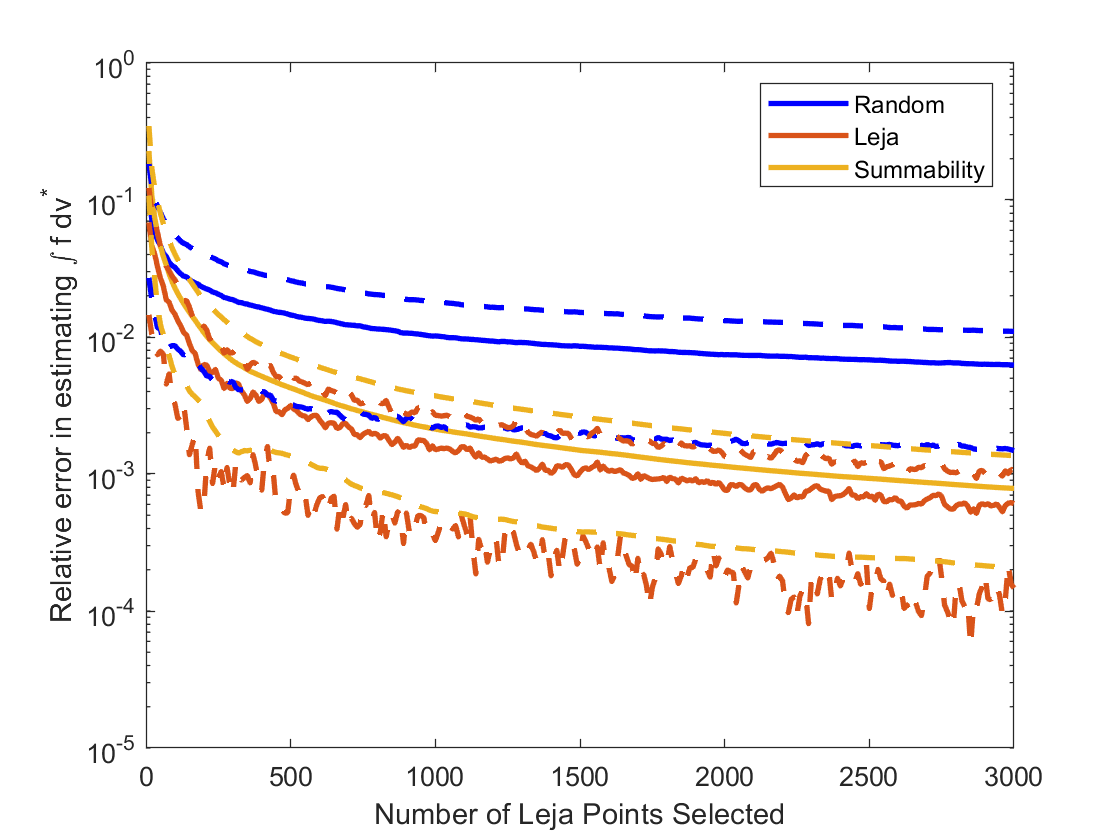} \\
Data and Function Example & Fast spectral decay $\tau=10$\\
\includegraphics[width=.4\textwidth]{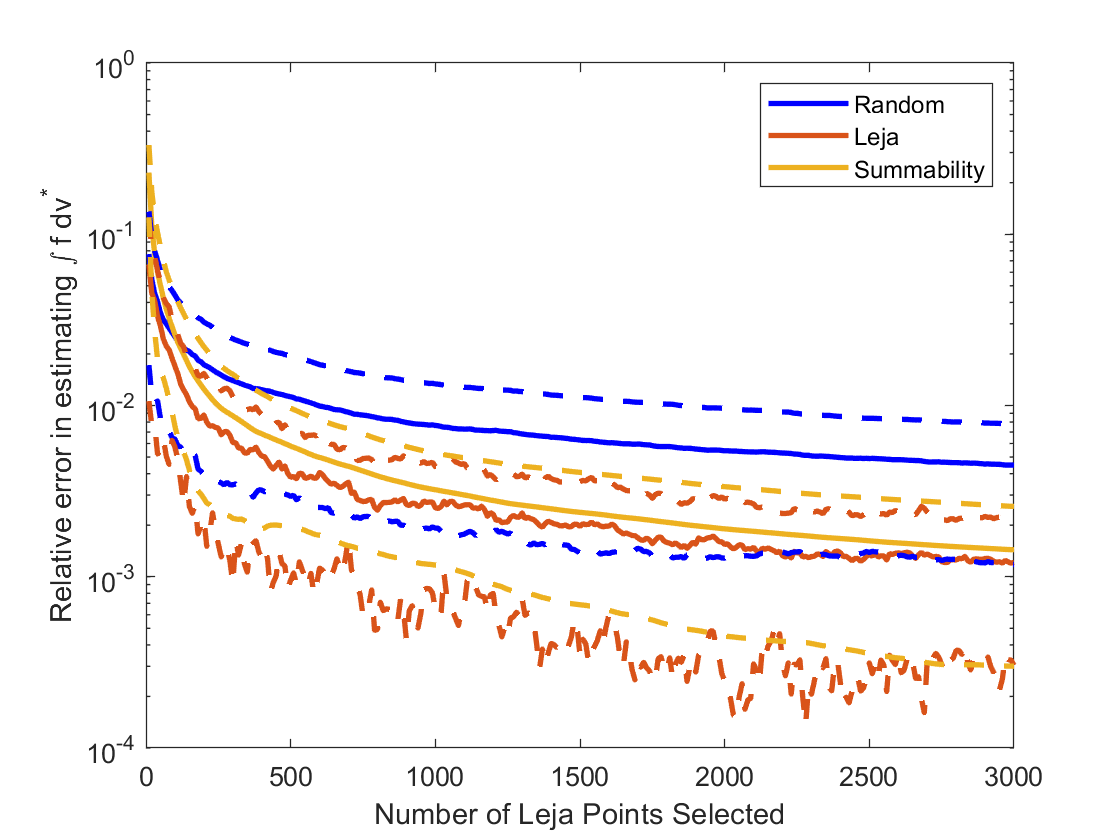} &
\includegraphics[width=.4\textwidth]{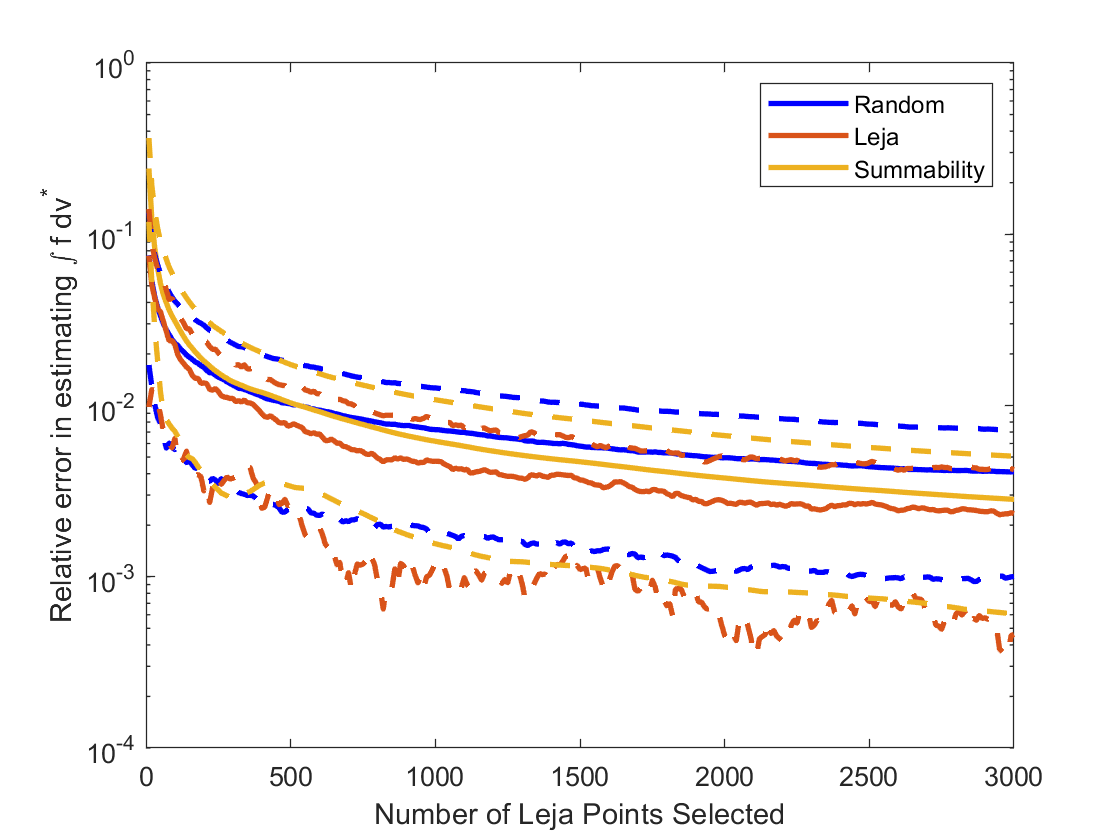}\\
Medium spectral decay $\tau=5$ & Slow spectral decay $\tau=1$
\end{tabular}
\caption{Two cluster metric graph example with non-bandlimited function as in Eq. \eqref{eq:spectraldecay}.  We display the quadrature error results over 1000 instantiations of the random graph/function, and for varying $\tau$. The dotted lines correspond to the confidence intervals around each mean curve of the same color.}\label{fig:f_tail}
\end{figure}

Figure \ref{fig:f_tail} shows an reduction in the gap between the Monte Carlo quadrature approximation rate and the Leja point approximation rate as the spectral decay tail increases.  This is to be expected, as a larger spectral tail reduces the smoothness of the function with respect to the graph.  But this experiment does show that  Leja point quadrature approximation rate does have a benefit over random sampling without the requirement that the target function be  spectrally band-limited.

\subsection{Real-World Graphs}\label{bhag:realworld}
Next, we illustrate our theory using a couple of real world graphs and functions.  
Since the graph is fixed in these examples and cannot be regenerated, we compute the Leja sequence for $1000$ different random choices of $a_0$, and average the errors.  Similarly for the Monte Carlo comparison, we generate $1000$ different sequences of sampled points.

In this first set of experiments, we examine the quadrature error for labels from the CORA data set \cite{sen2008collective}. 
 The standard form of the data set comprises a digraph with 2708 publications as vertices, and edge from $i$ to $j$ means that paper $i$ cited paper $j$. These publications are from seven areas of computer science and information theory, which corresponds to the 7 different classes in the data set.
There are a large number of weakly connected components, many of which are singletons.  The largest weakly connected component has $2485$ nodes.  
We treat the citation links as undircted edges and construct a binary, symmetric adjacency matrix $A$, and take the largest connected component as the graph.  This is equivalent to the largest weakly connected component, which results in an undirected graph with $2485$ vertices.

 Due to the sparsity of the number of edges in this graph, 
using the matrix $G$ constructed directly from the graph results in a slow rate at which $\sum_k G(\cdot, a_k)$ changes as the number of Leja points increases. 
The process is expedited if we extend the connections between documents to two steps away by taking $B=A^2$.  
In this experiment in Figure \ref{fig:cora}, we take the function $f$ to be the indicator function of class $i$ for $i=1, ..., 7$.  These labels represent the field of the given document, and are related to the network due to the obvious fact that papers in the same field are more likely to cite one another.

\begin{figure}
\centering
\footnotesize
\begin{tabular}{cccc}
\includegraphics[width=.2\textwidth]{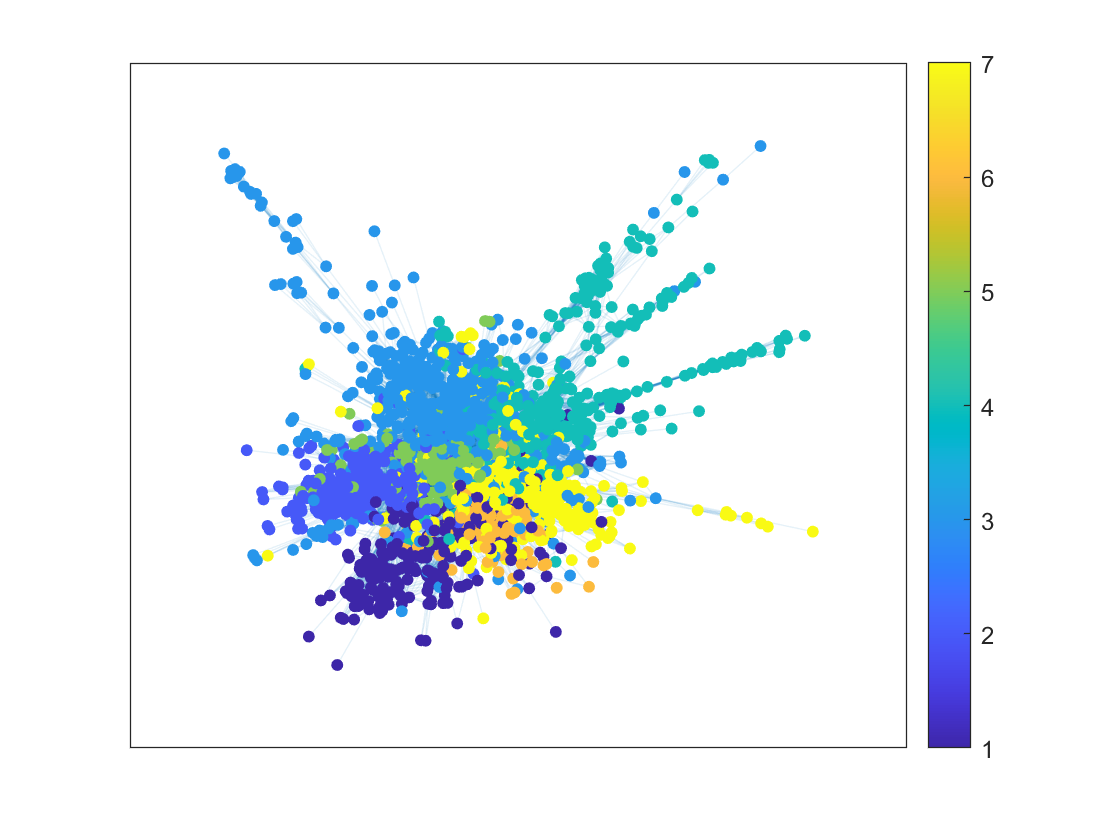} & 
\includegraphics[width=.2\textwidth]{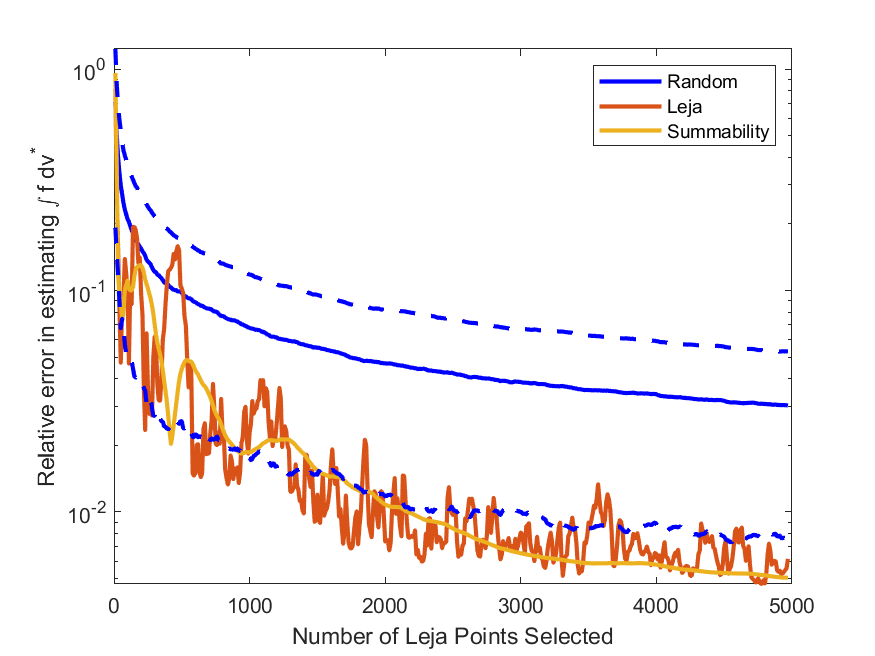} & 
\includegraphics[width=.2\textwidth]{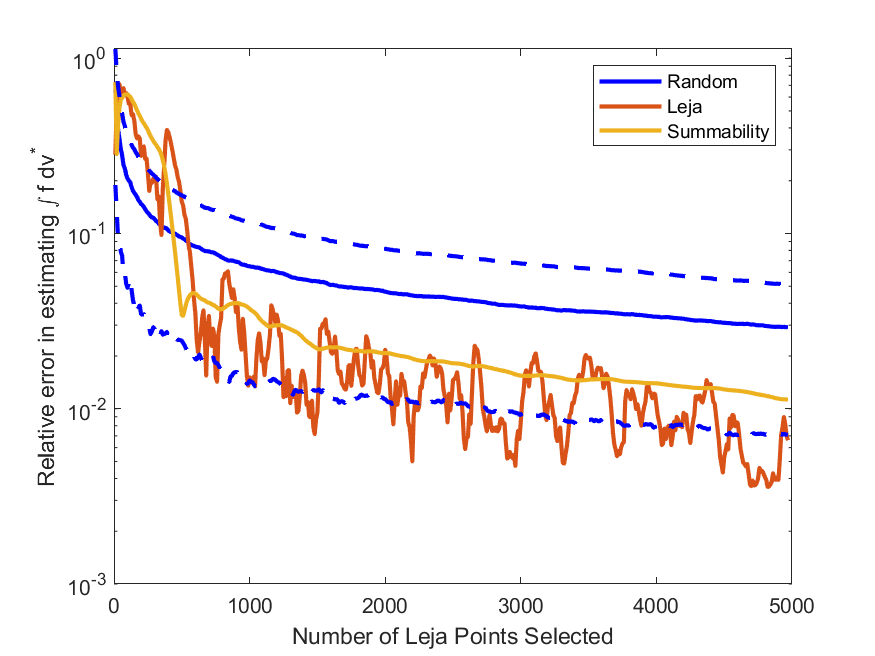} & 
\includegraphics[width=.2\textwidth]{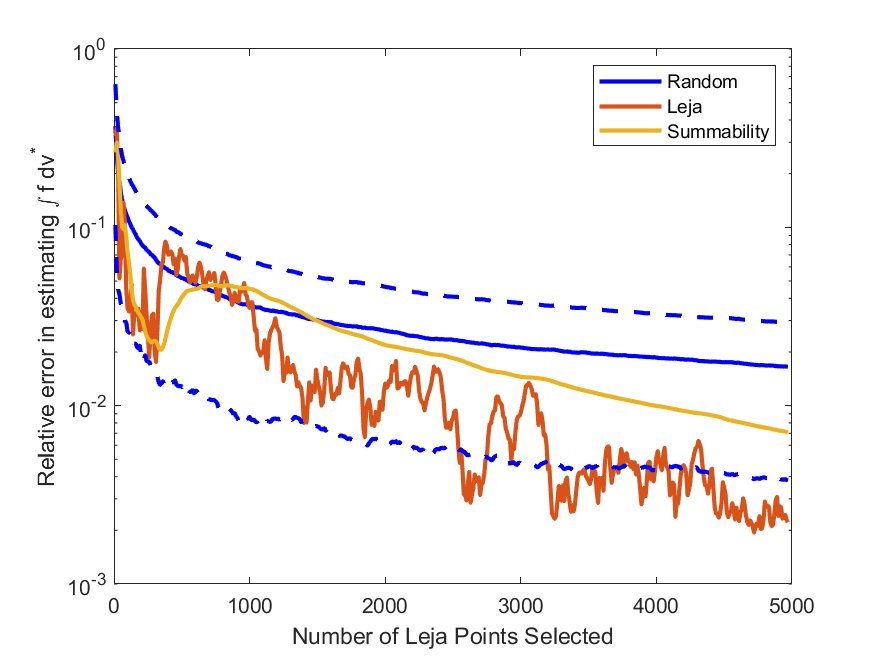} \\
Data and Labels & Label 1 & Label 2 & Label 3\\
\includegraphics[width=.2\textwidth]{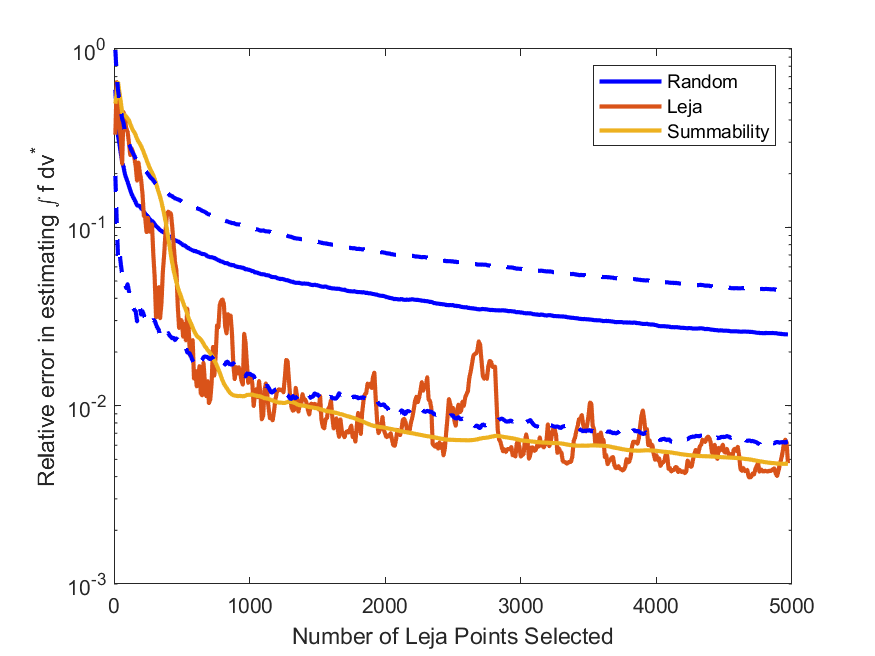} &
\includegraphics[width=.2\textwidth]{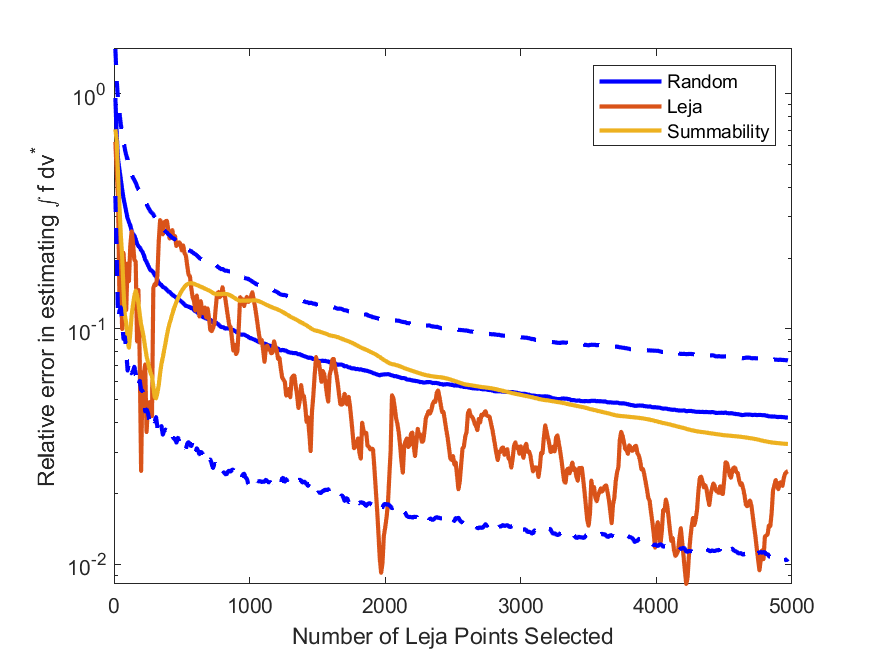} & 
\includegraphics[width=.2\textwidth]{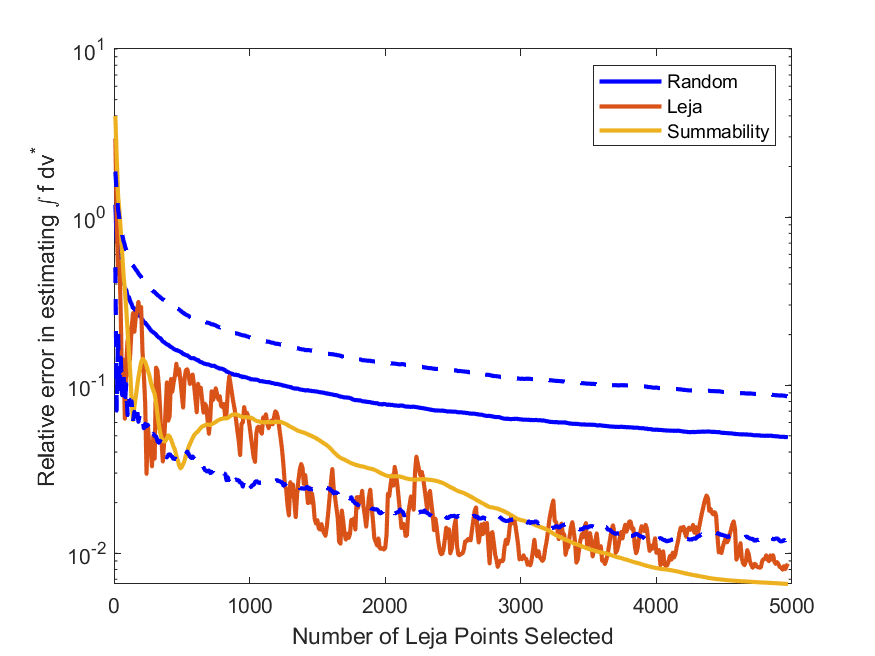} & 
\includegraphics[width=.2\textwidth]{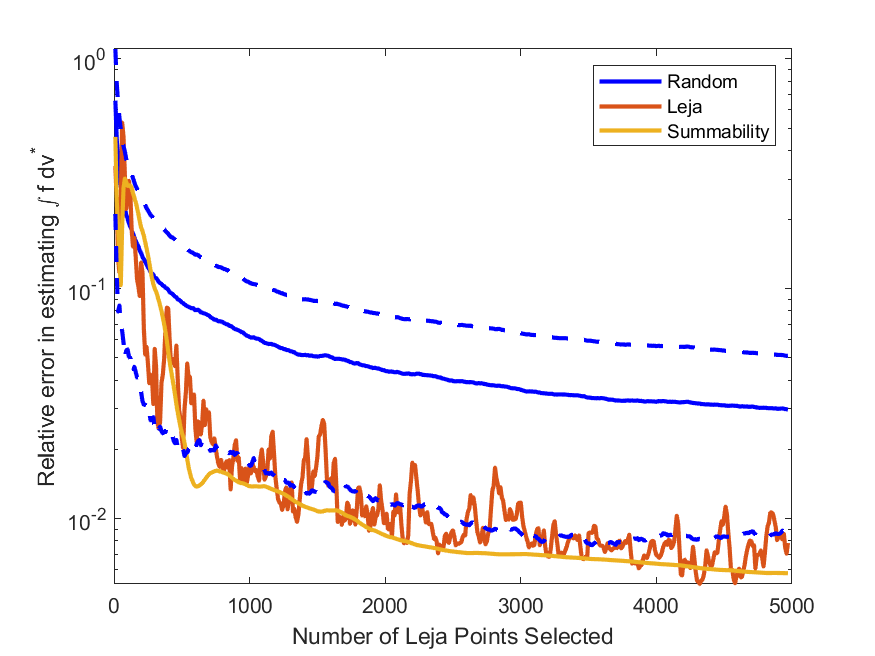}  \\
Label 4 & Label 5 & Label 6 & Label 7
\end{tabular}
\caption{Cora data set quadrature error for indicator function of various labels.  The Leja and summability curves are averaged across 1000 different initializations.  The dotted lines around the Monte Carlo curve are the confidence intervals. 
}\label{fig:cora}
\end{figure}

In the second set of experiments, we examine the exepcted value approximation for positive association with various propositions in the Proposition dataset \cite{smith2013role,zhu2014tripartite}.  The November 2012 California ballot contained 11 initiatives, or propositions, on a variety of issues, including state taxation, corrections, and food labelling among others. The data consist of Twitter posts related to initiatives, grouped according to different propositions. For each proposition, the data is a directed graph with edge from $i$ to $j$ if the tweet originated from user $i$ to user $j$. 
 The authors of \cite{smith2013role} have assigned an evaluation of the emotion of the sender with each tweet. We treat the edges as undirected and construct a binary, symmetric adjacency matrix $A$.
 We take the largest connected component of the graph for these experiments.  Each proposition has its own graph, with a positive or negative label of sentiment on each node.  The mean binary sentiment estimate could be used as a proxy for the number of people that would vote for/against the proposition.
 
 Due to the sparsity of the number of edges in this graph, we extend the connections between documents to two steps away by taking $B=A^2$.  
In this experiment, we take the function to be a binary indicator function ($0$ or $1$) of whether the sentiment is positive.   We examine two different graphs in Figure \ref{fig:prop}, those for Proposition 30 and those for Proposition 37, as those are the graphs with the largest number of vertices, with $4436$ vertices for Prop 30 and $8039$ vertices for Prop 37.  

As we can see in both the Cora experiments and the Proposition experiments, the Leja quadrature error performs better than the average Monte Carlo performance even for a small number of sampled points.  We also wish to address the relative lack of smoothness of the quadrature approximations relative to the synthetic experiments and the Monte Carlo trials.  First, recall that the labels here are incredibly non-smooth as they are binary.  Second, we note that the Leja point sampling can be thought of as a distance sampling scheme, choosing consecutive points far apart from one another and points likely to have different labels.  Because of these points, the approximation error will have a larger local variability than random sampling, even when averaged across multiple initializations.  This is also a motivation for using the summability approximation, which applies a decaying weight to added points and does not demonstrate the same local fluctuations as the unweighted Leja point averages.

\begin{figure}
\centering
\footnotesize
\begin{tabular}{cc}
\includegraphics[width=.4\textwidth]{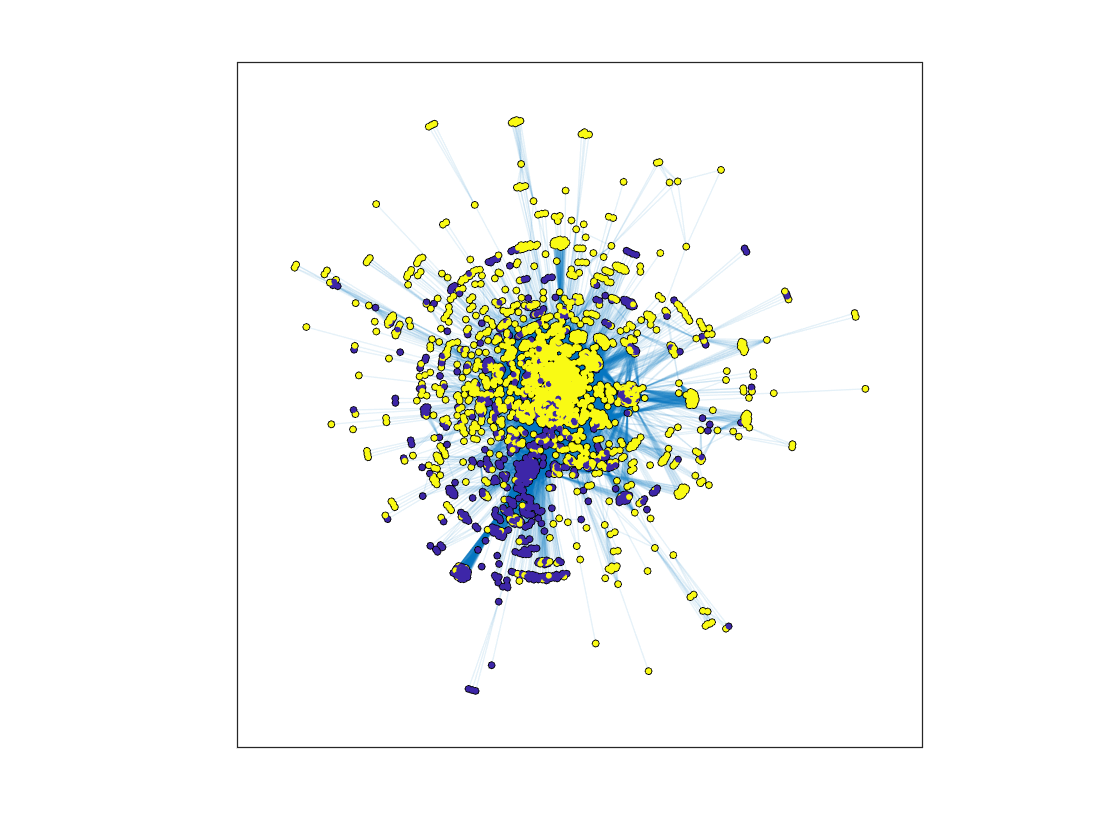} & 
\includegraphics[width=.4\textwidth]{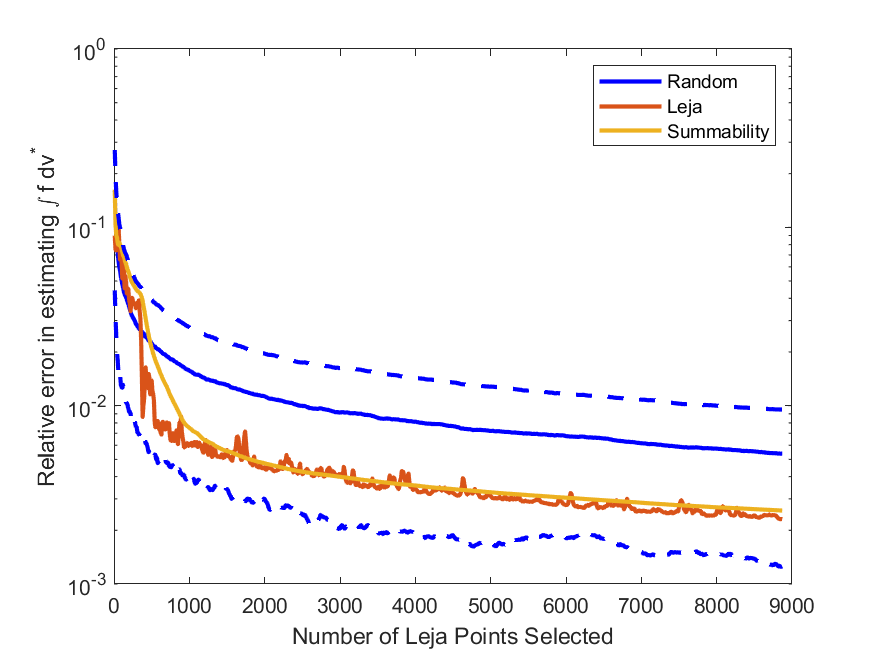} \\
(a) & (b)\\
\includegraphics[width=.4\textwidth]{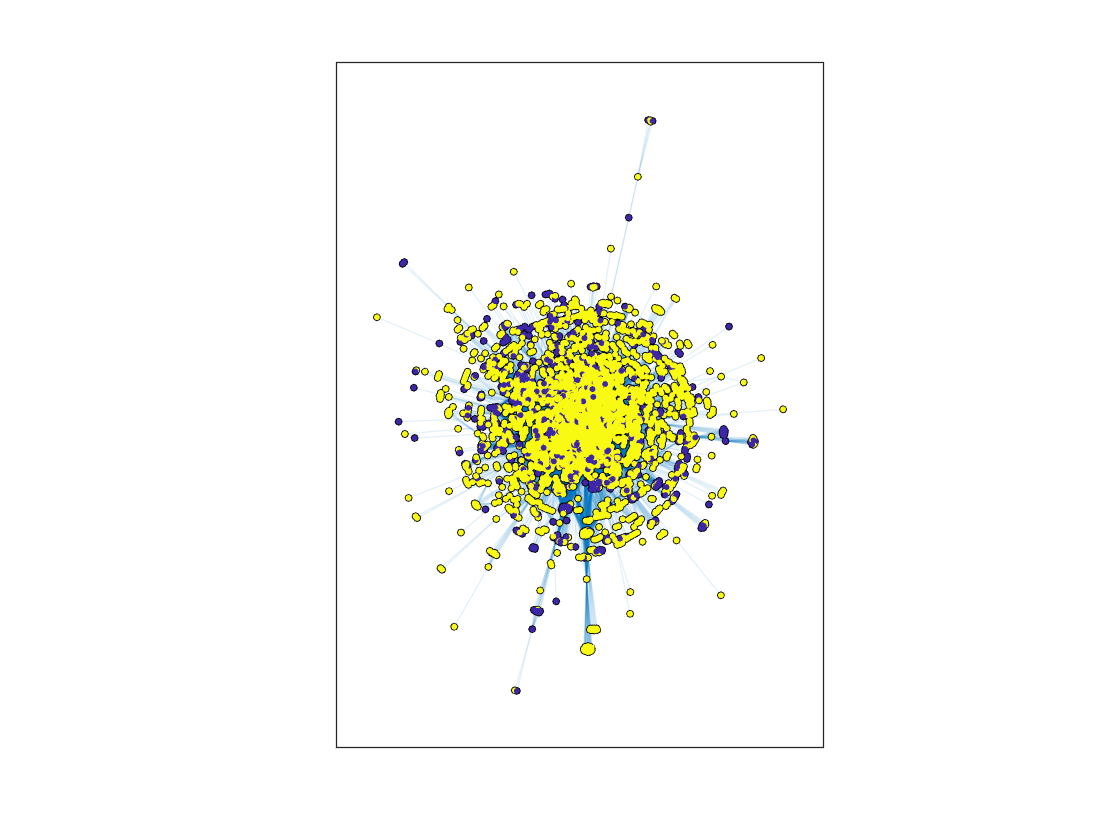} & 
\includegraphics[width=.4\textwidth]{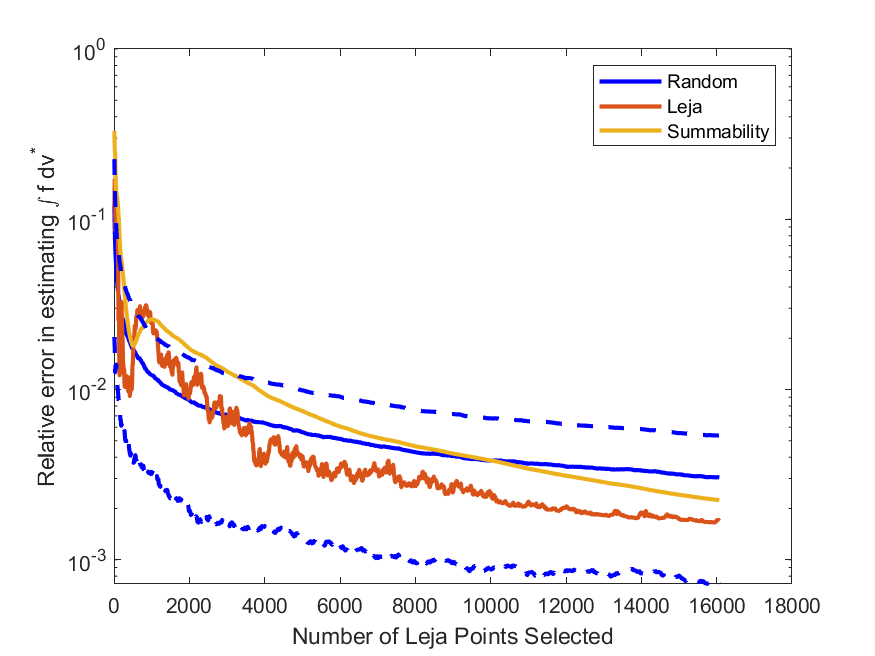} \\
(c) & (d)
\end{tabular}
\caption{Proposition data set quadrature error.  (a)-(b) Proposition 30, (c)-(d) Proposition 37.  The Leja and summability curves are averaged across 1000 different initializations.  The dotted lines around the Monte Carlo curve are the confidence intervals.}\label{fig:prop}
\end{figure}

We also demonstrate that the location of the Leja points for each of these data sets, using $\nu^*$ to be an inverse density estimate, in Figure \ref{fig:leja_real}.
\begin{figure}
\centering
\footnotesize
\begin{tabular}{ccc}
\includegraphics[width=.3\textwidth]{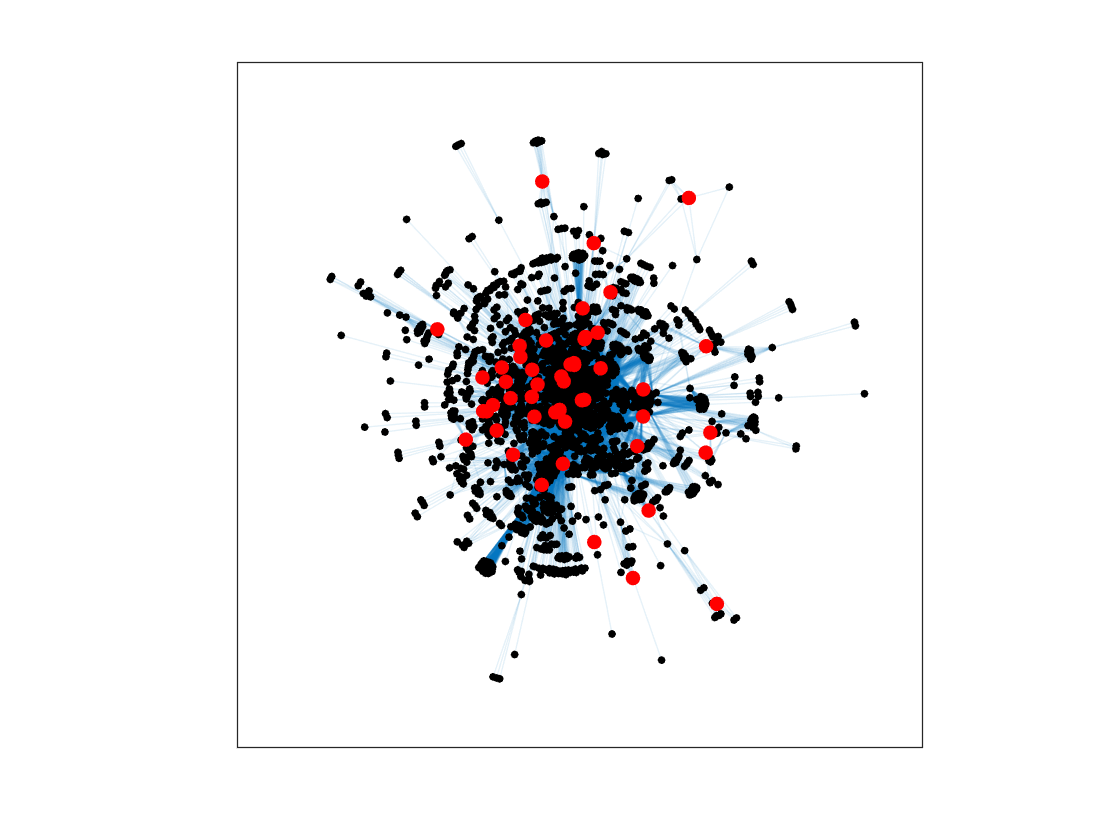} & 
\includegraphics[width=.3\textwidth]{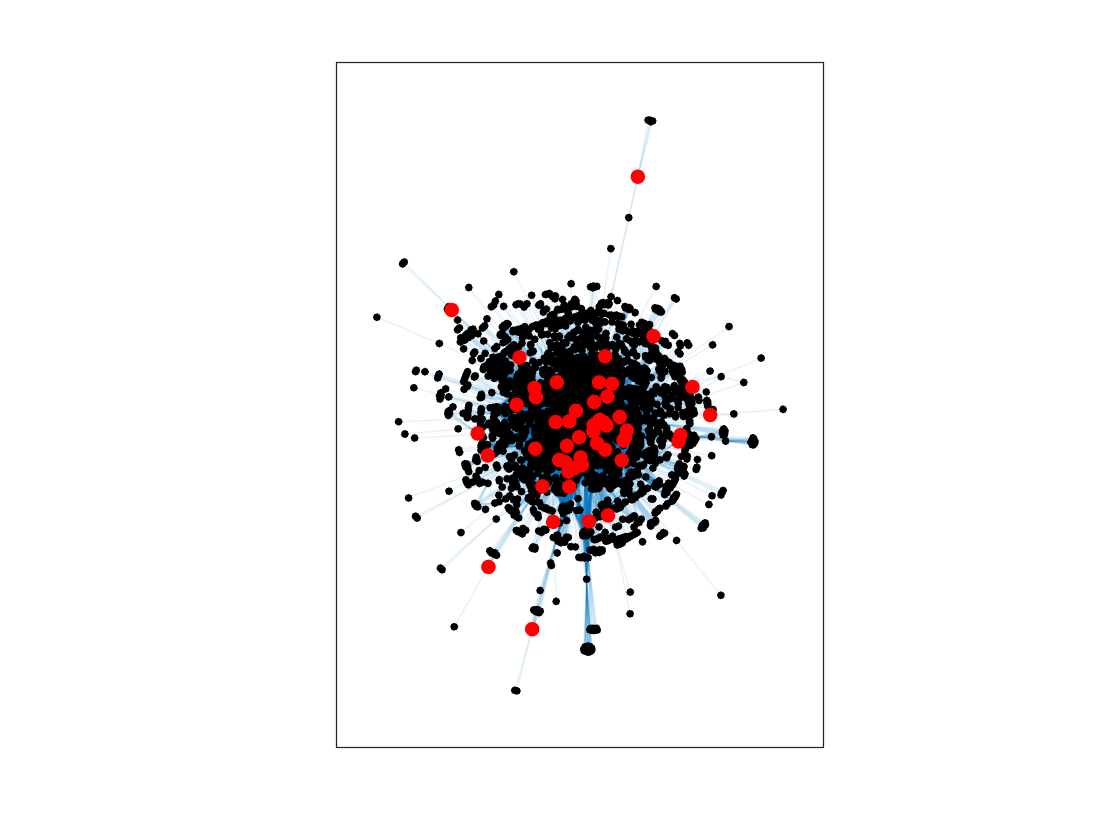} & 
\includegraphics[width=.3\textwidth]{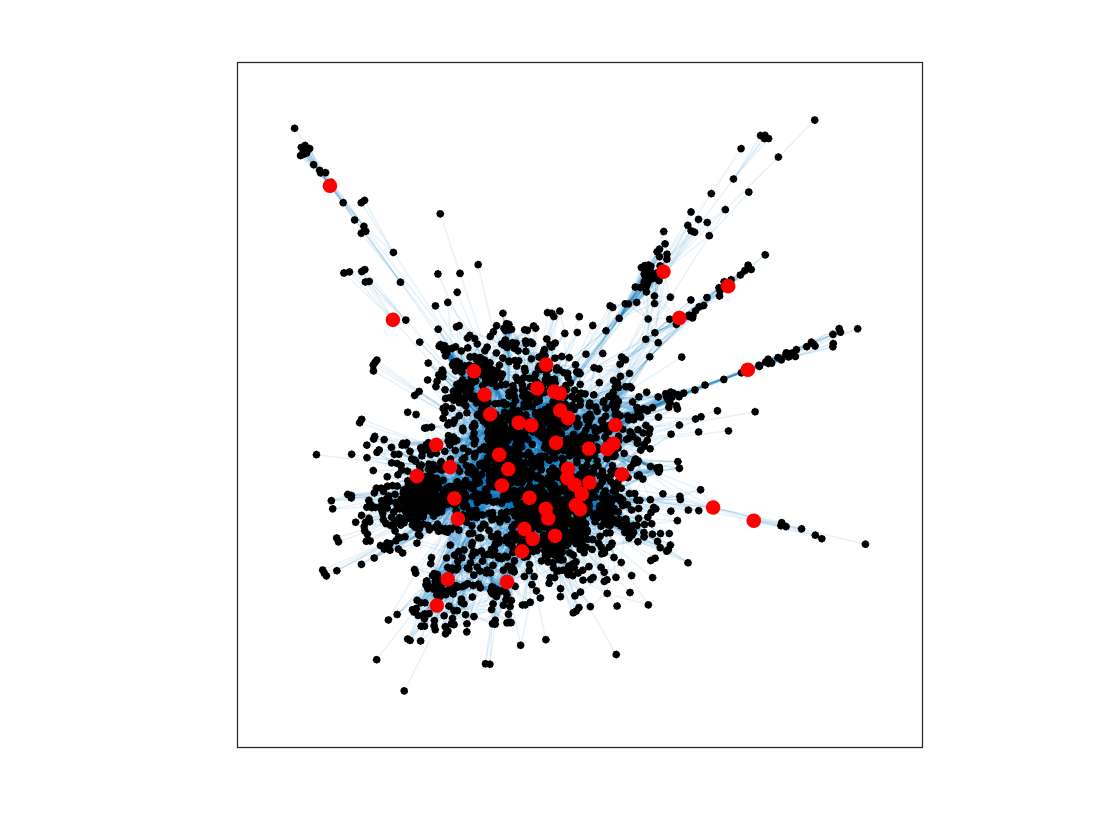} \\
Prop 30 & Prop 37 & Cora
\end{tabular}
\caption{Layout of 50 Leja points for various real-world data sets, with $\nu^*$ being the inverse of the kernel density estimate.}\label{fig:leja_real}
\end{figure}

\subsection{Additional Comparisons}

In this section, we compare our method to a natural alternative: computing a function interpolation from the sampled subset $W\subset V$ to the entire set of vertices $V$, and then integrating that estimate against $\nu^*$.  A recent method for this function interpolation is graph  sampling set selection (GSSS)  \cite{sakiyama2019eigendecomposition}, which is built on taking a weighted linear combination of columns of the spectrally filtered graph Laplacian.  We compare to this method in Figure \ref{fig:gsss} on the community graph data set proposed in \cite{sakiyama2019eigendecomposition}.  We compare using both the original adjacency matrix $A$, and using the two-hop adjacency matrix $A^2$, as the starting set of edges.  We take $\nu^*$ to be the inverse density of each community, normalized.
Again, we compare across 1000 instantiations of the graph with $N=1000$ points, and report the mean and standard deviation of the error at each fraction of points kept.   The function being regressed in each instantiation is chosen to be constant on each community, with the value randomly choosen from a uniform distribution.  

\begin{figure}
\centering
\begin{tabular}{cc}
\includegraphics[width=.4\textwidth]{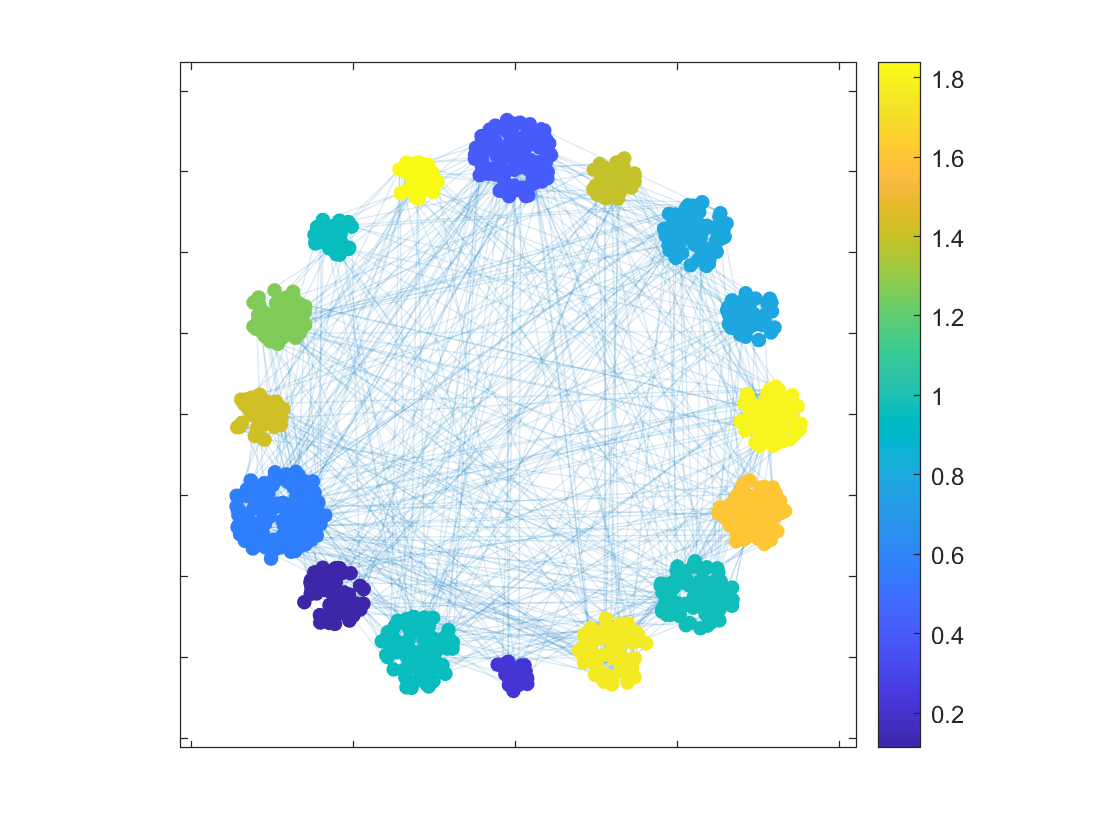} & 
\includegraphics[width=.4\textwidth]{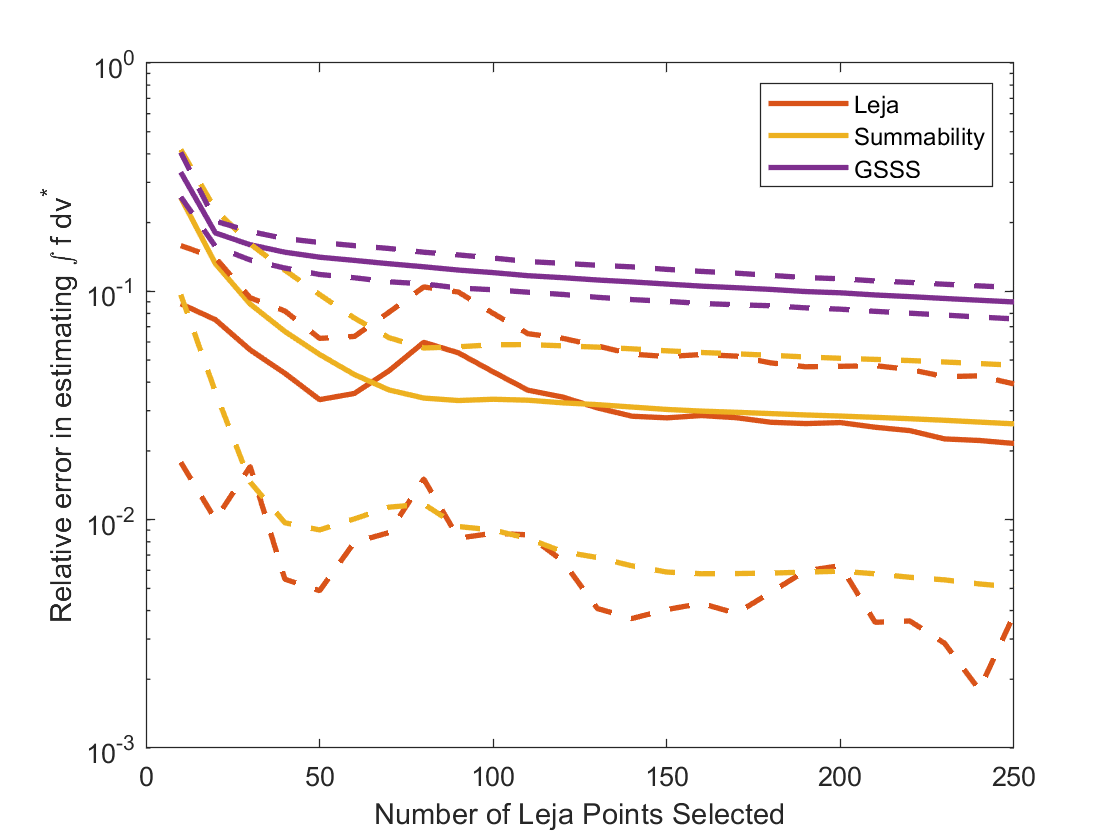} \\
(a) & (b)\\
\includegraphics[width=.4\textwidth]{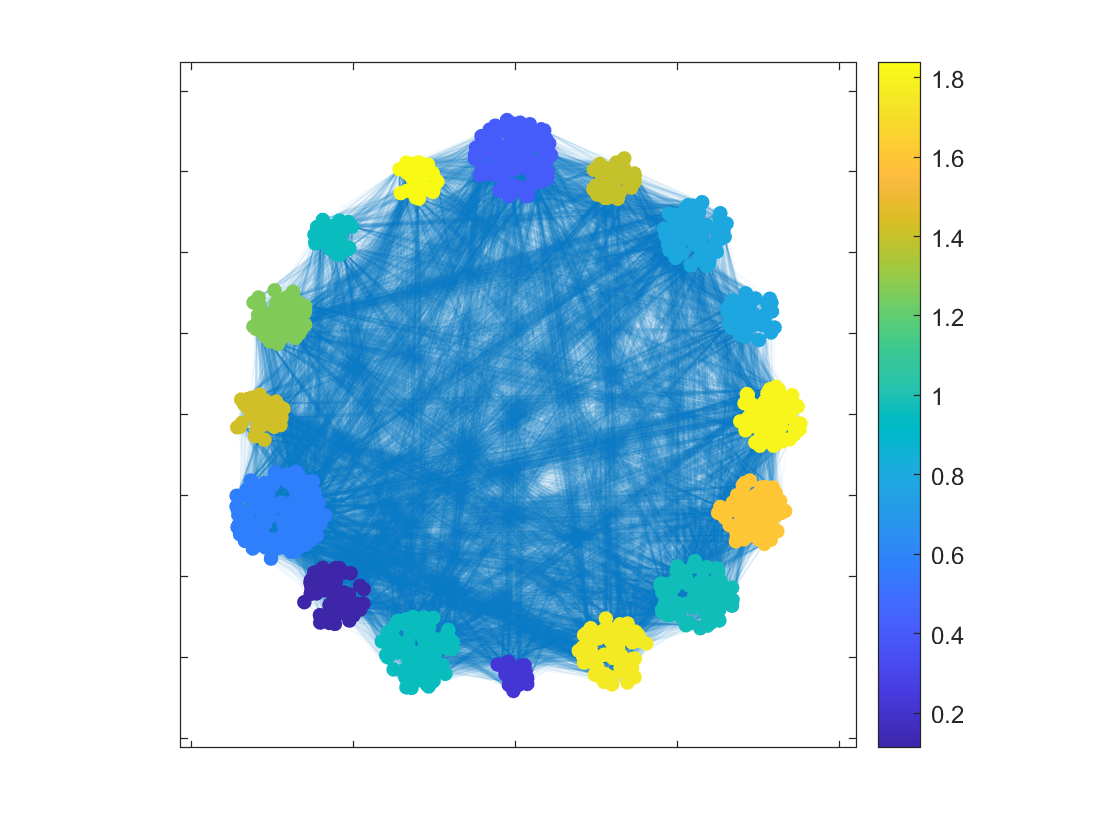} & 
\includegraphics[width=.4\textwidth]{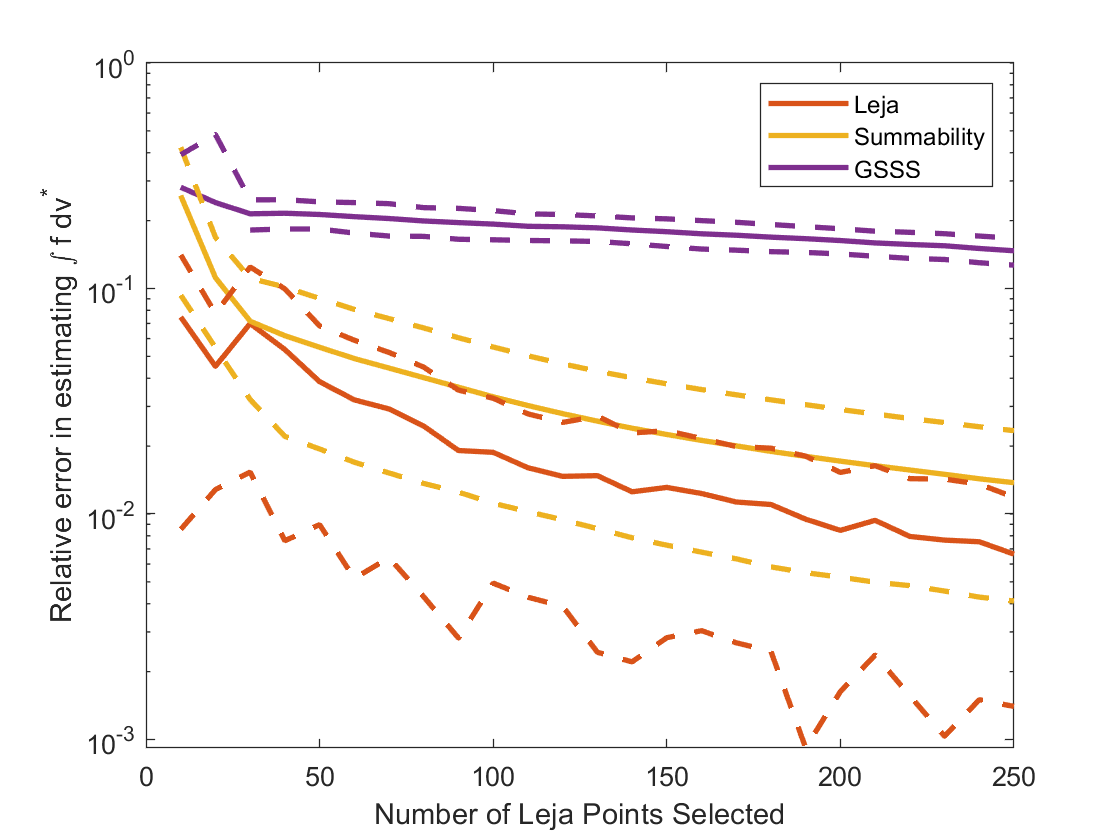}\\
(c) & (d)
\end{tabular}
\caption{(a)-(b) Adjacency of communities, (c)-(d) Adjacency matrix squared.  We display both an example of the graph/function, and the quadrature error results over 1000 instantiations of the random graph. The dotted lines correspond to the confidence intervals around each mean curve of the same color.}\label{fig:gsss}
\end{figure}

Finally, we demonstrate that the  bound \eqref{quadest} on the   quadrature error is independent of the number of data points for problems of fixed complexity.  In this experiment, we return to the two cluster example from Figure \ref{fig:twocluster} while varying the number of data points.  The graph created from this point cloud remains fixed, constructed with a Gaussian kernel of bandwidth $\sigma=0.1$.  In Figure \ref{fig:varyN}, we vary the number of data points and plot the quadrature error for a fixed number of Leja iterations.  After an initial level-off as the number of points grows, the error is more or less constant.  This shows that the error does not have a strong dependence on $N$ for neighbor graphs built from a fixed sampling distribution.
\begin{figure}
\centering
\footnotesize
\begin{tabular}{cc}
\includegraphics[width=.4\textwidth]{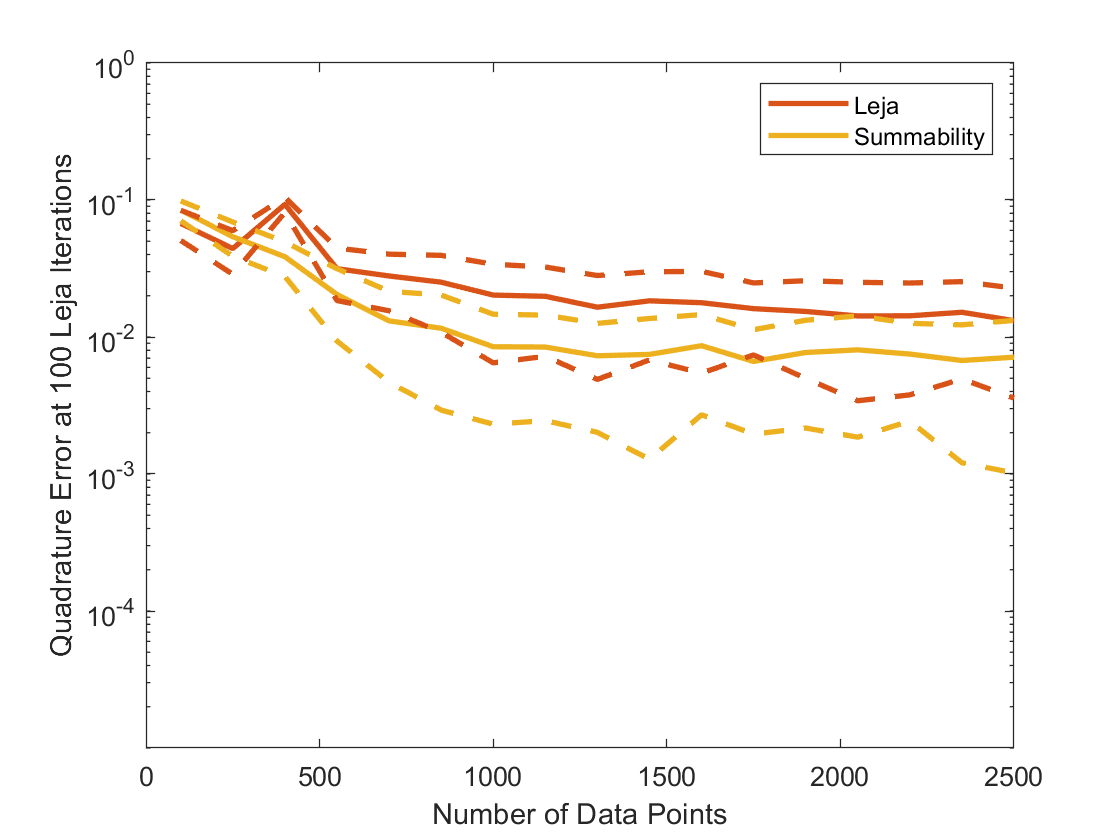} &
\includegraphics[width=.4\textwidth]{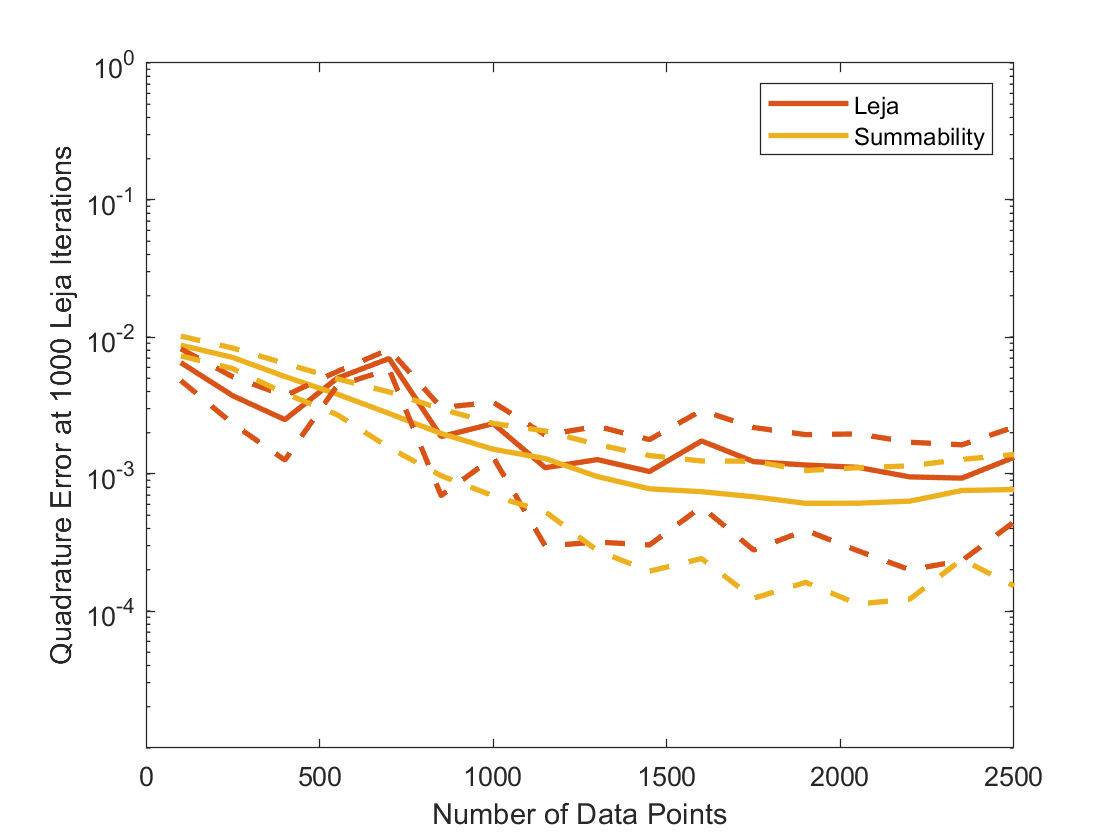} \\
100 Leja Iterations & 1000 Leja Iterations
\end{tabular}
\caption{Two cluster dataset with a fixed number of Leja points and a fixed Gaussian kernel bandwidth $\sigma=0.1$.  The quadrature error results are over 1000 instantiations of the random graph. The dotted lines correspond to the confidence intervals around each mean curve of the same color.}\label{fig:varyN}
\end{figure}

\bhag{Proofs}\label{bhag:proofs}

\noindent
\textsc{Proof of Lemma~\ref{lemma:frostman}.}\\

Let $\nu^*$ be an equilibrium measure. 
If $\nu\in\mathcal{P}$, then for all $t\in [0,1]$, $(1-t)\nu^*+t\nu\in\mathcal{P}$, and therefore, $t=0$ is a minimum for
$$
f(t)=G((1-t)\nu^*+t\nu,(1-t)\nu^*+t\nu)=(1-t)^2G(\nu^*,\nu^*)+t^2G(\nu,\nu)+2t(1-t)G(\nu,\nu^*).
$$
So, $f'(0)=-2G(\nu^*,\nu^*)+2G(\nu,\nu^*)\ge 0$. This implies \eqref{frostmanlow}.
Let $\epsilon>0$ and, in this proof only,  $L=\{x\in \XX : G(x,\nu^*)>\Gamma(G)+\epsilon\}$. Then
$$
\Gamma(G)=G(\nu^*,\nu^*)=\int_L G(x,\nu^*)d\nu^*+\int_{\XX\setminus L} G(x,\nu^*)d\nu^*(x) \ge \nu^*(L)(\Gamma(G)+\epsilon)+\Gamma(G)\nu^*(\XX\setminus L)=\Gamma(G)+\nu^*(L)\epsilon.
$$
Therefore, $\nu^*(L)=0$. Since $\XX$ is a finite set, this implies \eqref{frosmaneq}.

Next, let $G$ be conditionally positive semi-definite, and $\mu$ satisfy the conditions of part (b).
Then \eqref{tentfrosmaneq} shows that $G(\mu,\mu)=c$, and \eqref{tentfrostlow} shows that $G(\nu^*,\mu)\ge c$.
So,
$$
0\le G(\nu^*-\mu,\nu^*-\mu)=G(\nu^*,\nu^*)+G(\mu,\mu)-2G(\nu^*,\mu)\le G(\nu^*,\nu^*)-c.
$$
Thus, $c=G(\mu,\mu)\le G(\nu^*,\nu^*)$. Since $\nu^*$ is an equilibrium measure, it follows that $c=G(\mu,\mu)=G(\nu^*,\nu^*)=\Gamma(G)$, so that $\mu$ is also an equilibrium measure.
Also, $G(\nu^*-\mu,\nu^*-\mu)=0$, so that $\mu-\nu^*$ is in the null space of $G$.
\qed

The next theorem shows in a standard manner than the sequence $\sigma_n$ associated with the  Leja points converge to the equilibrium measure.

\begin{theorem}\label{theo:lejaconverge}
We have
\be\label{mmdist}
G(\sigma_n-\nu^*, \sigma_n-\nu^*)\le \frac{1}{n}\left(M-G(\nu^*,\nu^*)\right), \qquad n\in\ZZ_+.
\ee
Any weak-star limit of a subsequence of $\{\sigma_n\}$ is also an equilibrium measure. 
\end{theorem}

\begin{proof}\ 
 We have for $n\ge 2$,
\be\label{pf1eqn1}
\begin{aligned}
G(\sigma_n,\sigma_n)&=\frac{1}{n^2}\sum_{j,k=0}^{n-1}G(\nu_j,\nu_k)=\frac{1}{n^2}\sum_{j=0}^{n-1}G(\nu_j,\nu_j)+\frac{2}{n^2}\sum_{0\le j<k\le n-1}G(\nu_j,\nu_k)\\
&=\frac{1}{n^2}\sum_{j=0}^{n-1}G(\nu_j,\nu_j)+\frac{2}{n^2}\sum_{k=1}^{n-1} G\left(\sum_{j=0}^{k-1}\nu_j,\nu_k\right)=\frac{1}{n^2}\sum_{j=0}^{n-1}G(\nu_j,\nu_j)+\frac{2}{n^2}\sum_{k=1}^{n-1} kG(\sigma_k,\nu_k).
\end{aligned}
\ee
In view of the definition \eqref{lejadefrecbis} of Leja points,
\be\label{pf1eqn3}
G(\sigma_k,\nu_k)=G(\nu_k,\sigma_k)\le G(x,\sigma_k), \qquad x\in\XX.
\ee
Integrating both sides with respect to $\nu^*$ and using \eqref{frostman}, we get for all $k\in\ZZ_+$,
$$
G(\sigma_k,\nu_k)\le G(\sigma_k,\nu^*)=G(\nu^*,\nu^*).
$$
Therefore, \eqref{pf1eqn1} leads to
\be\label{pf1eqn2}
G(\sigma_n,\sigma_n)\le M/n+\frac{2}{n^2}G(\nu^*,\nu^*)\sum_{k=1}^{n-1} k =G(\nu^*,\nu^*)+(M-G(\nu^*,\nu^*))/n.
\ee
In view of \eqref{frostman}, 
$$
G(\sigma_n-\nu^*, \sigma_n-\nu^*)=G(\sigma_n,\sigma_n)+G(\nu^*,\nu^*)-2G(\sigma_n,\nu^*)=G(\sigma_n,\sigma_n)-G(\nu^*,\nu^*).
$$
Therefore, \eqref{pf1eqn2} implies \eqref{mmdist}. 
In turn, if $\{\sigma_n\}_{n\in\Lambda}$ be any subsequence of $\{\sigma_n\}$ converging weak-star to a probability measure $\sigma$, then \eqref{mmdist} shows that $G(\sigma-\nu^*, \sigma-\nu^*)=0$. 
Therefore, $\sigma$ is also an equilibrium measure.
\end{proof}

We note a couple of corollaries of the proof of Theorem~\ref{theo:lejaconverge}.
\begin{cor}\label{cor:lejasummability}
For $n\ge 2$,
\be\label{lejasummability}
\frac{n}{n-1}G(\nu^*,\nu^*)-M/n \le \frac{2}{n(n-1)}\sum_{k=1}^{n-1} kG(\nu_k,\sigma_k)\le G(\nu^*,\nu^*).
\ee
\end{cor}
We formulate the next corollary as a theorem in its own right.
\begin{theorem}\label{theo:potentialbd}
Let $G(\nu^*,\nu^*)\ge 0$. Then for all $x\in\XX$, $n\ge 2$,
\be\label{potentialbd1}
G(\nu^*,\nu^*)-M/n \le \frac{2}{n(n-1)}\sum_{k=1}^{n-1} kG(x,\sigma_k)\le \frac{2\tn G\tn}{n}\le G(\nu^*,\nu^*)+\frac{2\tn G\tn}{n}.
\ee
In particular,
\be\label{potentialbd}
\max_{x\in\XX}\left|\frac{2}{n(n-1)}\sum_{j=0}^{n-1}(n-j)G(x,\nu_j)-G(x,\nu^*)\right|\le \frac{2\tn G\tn}{n}.
\ee
\end{theorem}

\begin{proof}\ 
The first estimate in \eqref{potentialbd1} follows from the first estimate in \eqref{lejasummability} and \eqref{pf1eqn3}. 
The second estimate in \eqref{potentialbd1} follows from the fact  that
\be\label{pf4eqn2}
G(x,\sigma_k)=\frac{1}{k}\sum_{j=0}^{k-1}G(x,\nu_k)=\frac{1}{k}\sum_{j=0}^{k-1}G(x,a_k)\le \frac{1}{k}\sum_{j=0}^{k-1}|G(x,a_k)|\le \frac{2\tn G\tn}{k}.
\ee
The final inequality in \eqref{potentialbd1} following from the fact that $G(\nu^*,\nu^*) \ge 0$.
We observe now that 
\be\label{pf4eqn1}
\sum_{k=1}^{n-1}kG(x,\sigma_k)=\sum_{k=0}^{n-1}kG(x,\sigma_k)=\sum_{k=0}^{n-1}\sum_{j=0}^{n-1}G(x,\nu_j)=\sum_{j=0}^{n-1}(n-j)G(x,\nu_j).
\ee
Therefore, \eqref{potentialbd} follows from \eqref{potentialbd1} and the fact that $G(x,\nu^*)=G(\nu^*,\nu^*)$ for all $x\in\XX$.
\end{proof}

\begin{theorem}\label{theo:erdosturanimproved}
If $G(x,y)\ge 0$ for all $x,y\in\XX$, we have for $n\ge 2$,
\be\label{erdosturanimproved}
\max_{x\in\XX}|G(x,\nu^*)-G(x,\sigma_n)|\le \frac{3\tn G\tn}{n+1}.
\ee
\end{theorem}

\begin{proof}\ 
For $\ell\ge 2$, we have
$$
\begin{aligned}
 \ell G(\nu_\ell,\sigma_\ell)&=\sum_{j=0}^{\ell-1}G(a_\ell,a_j) = G(a_\ell,a_{\ell-1})+ \sum_{j=0}^{\ell-2}G(a_\ell,a_j)=G(a_\ell,a_{\ell-1})+(\ell-1)G(\nu_\ell,\sigma_{\ell-1})\\
 &\ge (\ell-1)G(\nu_{\ell-1},\sigma_{\ell-1})\ge \ell G(\nu_{\ell-1},\sigma_{\ell-1})-\frac{1}{\ell-1}\sum_{j=0}^{\ell-2}G(a_{\ell-1},a_j) \ge \ell G(\nu_{\ell-1},\sigma_{\ell-1})-\frac{\tn G\tn}{\ell-1};
 \end{aligned}
 $$
i.e.,
$$
G(\nu_{\ell-1},\sigma_{\ell-1})-G(\nu_\ell,\sigma_\ell)\le \tn G\tn\left(\frac{1}{\ell-1}-\frac{1}{\ell}\right).
$$
Therefore, for any $k$ with $2\le k\le n$, a summation in the above inequality leads to
$$
G(\nu_k,\sigma_k)-G(\nu_n,\sigma_n)\le \tn G\tn\frac{n-k}{nk}.
$$
Since $G(\nu^*,\nu^*)\ge 0$, Corollary~\ref{cor:lejasummability} now shows that
$$
\begin{aligned}
\frac{n(n+1)}{2}\left(G(\nu^*,\nu^*)-\frac{M}{n+1}\right)&\le 
\sum_{k=1}^{n} kG(\nu_k,\sigma_k)\le \frac{n(n+1)}{2}G(\nu_n,\sigma_n)+ G(\nu_1,\sigma_1) +\frac{\tn G\tn}{n}\sum_{k=2}^{n}(n-k)\\
&=\frac{n(n+1)}{2}G(\nu_n,\sigma_n)+ G(\nu_1,\sigma_1) +\frac{\tn G\tn}{n}\frac{(n-1)(n-2)}{2}.
\end{aligned}
$$
Rearranging,
$$
G(\nu^*,\nu^*)-G(\nu_n,\sigma_n)\le \frac{M}{n+1}+\frac{2G(\nu_1,\sigma_1)}{n(n+1)}+\frac{\tn G\tn}{n}\frac{(n-1)(n-2)}{n(n+1)}.
$$
For $n\ge 2$, this gives 
$$
G(\nu^*,\nu^*)-G(\nu_n,\sigma_n)\le\frac{2M+\tn G\tn}{n+1}\le \frac{3\tn G\tn}{n+1}.
$$
For any $x\in \XX$, we have  $G(x,\nu^*)=G(\nu^*,\nu^*)$ and $G(\nu_n,\sigma_n)\le G(x,\sigma_n)$. Therefore,
$$
G(x,\nu^*)-G(x,\sigma_n)\le \frac{3\tn G\tn}{n+1}.
$$
In the reverse direction, we recall \eqref{pf4eqn2}.
Together with \eqref{potentialbd1},
this leads to \eqref{erdosturanimproved}.
\end{proof}

\noindent
\textsc{Proof of Theorem~\ref{theo:quadrature}.}\\

The estimate \eqref{quadest} follows easily from \eqref{erdosturanimproved} and  the observation that
$$
\int_\XX fd\nu^*-\frac{1}{n}\sum_{k=0}^{n-1}f(a_k)=\int_\XX fd\nu^*-\int_\XX fd\sigma_n =\int_\XX  \left\{\int_\XX G(x,y)d\nu^*(y)-\int_\XX G(x,y)d\sigma_n(y)\right\}d\mathcal{D}_G(f)(x).
$$
\qed

\section*{Acknowledgments}
The work of AC was supported in part by NSF DMS grants 2012266,  1819222, and Sage Foundation Grant 2196.
The work of HNM is supported in part
NSF DMS grant 2012355 and ARO grant W911NF2110218.
We thank  Professors Percus  at Claremont Graduate University for his help in securing the Proposition data set, which was sent to us by Dr. Linhong Zhu at USC Information Sciences Institute in Marina Del Ray, California.

\bibliographystyle{abbrv}
\bibliography{lejapts}
\end{document}